\documentclass{article}

\usepackage[preprint]{icml2026}

\usepackage{hyperref}
\usepackage{url}

\usepackage{booktabs}       

\usepackage{multirow}         
\usepackage{arydshln}
\usepackage{algorithm}
\usepackage{algorithmic}
\usepackage{graphicx} 
\usepackage{subcaption} 
\usepackage{amsmath}
\usepackage{amsfonts}       
\usepackage{amsthm}
\usepackage{amssymb}
\usepackage{xcolor} 

\theoremstyle{theorem}
\newtheorem{theorem}{Theorem}[section]
\newtheorem{lemma}[theorem]{Lemma}

\newtheorem{definition}[theorem]{Definition}
\newtheorem{assumption}[theorem]{Assumption}

\theoremstyle{remark}
\newtheorem{remark}[theorem]{Remark}

\newcommand{\mc}[1]{\mathcal{#1}}
\newcommand{\bs}[1]{\boldsymbol{#1}}
\newcommand{\bb}[1]{\mathbb{#1}}
\renewcommand\algorithmiccomment[1]{%
  \hfill\#\ {\textit{#1}}%
}

\usepackage{cleveref}
\usepackage[textsize=tiny]{todonotes}

\icmltitlerunning{AdaMeZO: Adam-style Zeroth-Order Optimizer for LLM Fine-tuning Without Maintaining the Moments}

\begin{document}

\twocolumn[
  \icmltitle{AdaMeZO: Adam-style Zeroth-Order Optimizer for LLM Fine-tuning \\ Without Maintaining the Moments}



  \icmlsetsymbol{equal}{*}

  \begin{icmlauthorlist}
    \icmlauthor{Zhijie Cai}{equal,yyy,zzz}
    \icmlauthor{Haolong Chen}{equal,yyy,zzz}
    \icmlauthor{Guangxu Zhu}{yyy,zzz,uuu}
  \end{icmlauthorlist}

  \icmlaffiliation{yyy}{Shenzhen Research Institute of Big Data}
  \icmlaffiliation{zzz}{The Chinese University of Hong Kong-Shenzhen}
  \icmlaffiliation{uuu}{Shenzhen Loop Area Institute}

  \icmlcorrespondingauthor{Guangxu Zhu}{gxzhu@sribd.cn}

  \icmlkeywords{Machine Learning, ICML}

  \vskip 0.3in
]



\printAffiliationsAndNotice{}  

\begin{abstract}
  Fine-tuning LLMs is necessary for various dedicated downstream tasks, but classic backpropagation-based fine-tuning methods require substantial GPU memory. To this end, a recent work, MeZO, which relies solely on forward passes to fine-tune LLMs, significantly reduces GPU requirements at the cost of slower convergence due to its indifference to loss landscapes. Standard solutions, such as Adam, explore loss landscapes by estimating the first- and second-order moments and storing them in memory to guide the model's movement through dimensions with lower curvature and vice versa. However, directly applying Adam negates MeZO's advantage as it will triple the memory requirement. In light of this, we propose AdaMeZO, a zeroth-order optimizer that leverages Adam-style first- and second-moment estimates without maintaining them in memory. We present a theoretical analysis of AdaMeZO, corroborated by extensive experiments demonstrating AdaMeZO's performance, showing that AdaMeZO can outperform MeZO while requiring up to $70\%$ fewer forward passes. Trajectory visualizations affirm AdaMeZO's ability to adapt to diverse loss landscapes. 
\end{abstract}

\section{Introduction}

\begin{table*}[h]
\caption{Key features for AdaMeZO and methods in comparison. $P$ in the first column denotes the amount of memory required to store the model weight, and $B \gg P$ denotes the amount of memory required to perform backpropagation. $\delta \ll 1$ is a small positive number. ``FP'' abbreviates forward pass.}
\begin{center}
\small
\label{table:memory}
\begin{tabular}{rcccc}
\toprule
\multicolumn{1}{l}{} & Param. memory    & FP per step & $1$st moment & \multicolumn{1}{l}{$2$nd moment} \\ \hline
Adam \cite{kingma2014adam}                & $3P+B$ & $1$                    & $\checkmark$      & $\checkmark$                          \\
MeZO \cite{malladi2023fine}                & $P$     & $2$                     & $\times$     & $\times$                         \\
HELENE \cite{zhao2024helene}              & $3P$    & $2$                     & $\checkmark$      & $\checkmark$                          \\
HiZOO \cite{zhao2024second}               & $2P$    & $3$                     & $\times$     & $\checkmark$                          \\
\hdashline
AdaMeZO                 & $\mathbf{(1+\delta)P}$ & $\mathbf{2}$                     & $\checkmark$      & $\checkmark$                          \\ \bottomrule
\end{tabular}
\end{center}
\end{table*}

Fine-tuning LLMs is necessary for specialized downstream tasks and has recently attracted significant attention. Many works have emerged that aim to tune models while accessing as little memory as possible. Popular first-order methods known as parameter-efficient fine-tuning (PEFT) to alleviate the heavy memory cost by modifying only a small (potentially extra) part of the whole model \cite{hu2022lora, li2021prefix, lester2021power, dettmers2023qlora, pan2024lisa}. Additionally, a zeroth-order method \cite{malladi2023fine} enables discarding backpropagation, the primary contributor to LLM fine-tuning's memory cost, making it accessible on resource-limited devices.

As shown in \Cref{table:memory}, MeZO features an SGD-style update rule \cite{rumelhart1986learning, bottou2018optimization}, allowing in-place parameter modification. After in-place model perturbation for gradient projection estimation, the gradients are not dumped into memory; instead, they are generated by a pseudo-random number generator (PRNG) before being scaled by the previously computed projection, reducing the memory cost for fine-tuning to the equivalent of deploying one. However, updating the model with only the most recent gradient estimate can lead to slower convergence, especially with noisy, isotropic zeroth-order gradient estimators. In comparison, adaptive optimizers like Adam \cite{kingma2014adam} and AdamW \cite{loshchilov2017decoupled}, which correct updates with preconditioners, are more widely adopted since the loss landscapes of LLMs exhibit complex curvature spectra across different dimensions, as documented in \cite{sagun2016eigenvalues, ghorbani2019investigation, zhang2023eva, das2024towards}.

\begin{figure*}[t]
    \centering
    \begin{minipage}[c]{0.32\textwidth}
        \centering
        \includegraphics[scale=0.31]{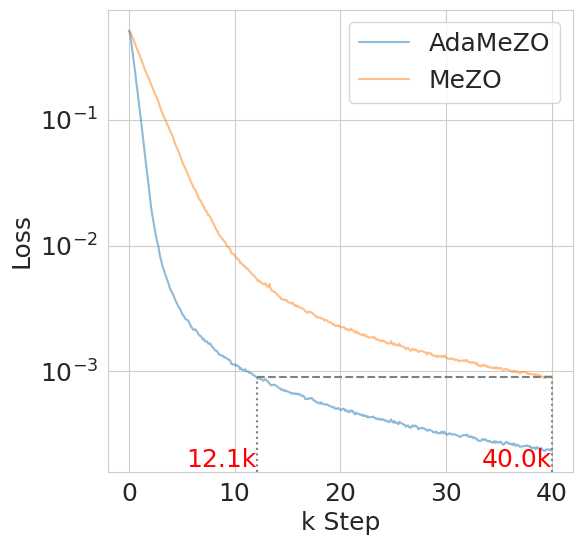}
        \subcaption{RoBERTa-Large SST-2.}
        \label{fig:loss_comparison1}
    \end{minipage}
    \begin{minipage}[c]{0.32\textwidth}
        \centering
        \includegraphics[scale=0.31]{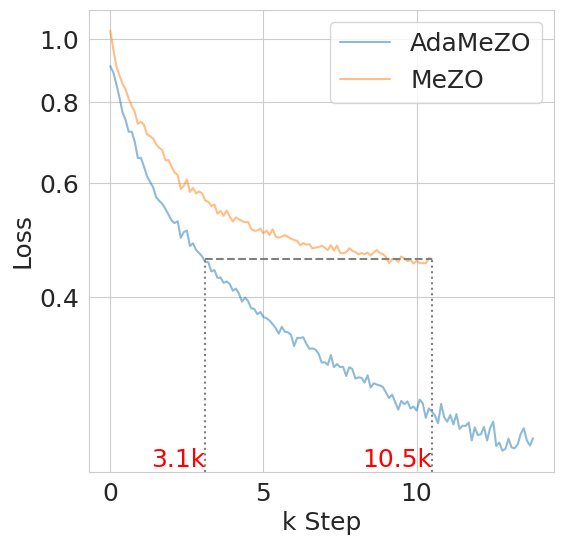}
        \subcaption{OPT-1.3b SST-2.}
        \label{fig:loss_comparison2}
    \end{minipage}
    \begin{minipage}[c]{0.32\textwidth}
        \centering
        \includegraphics[scale=0.31]{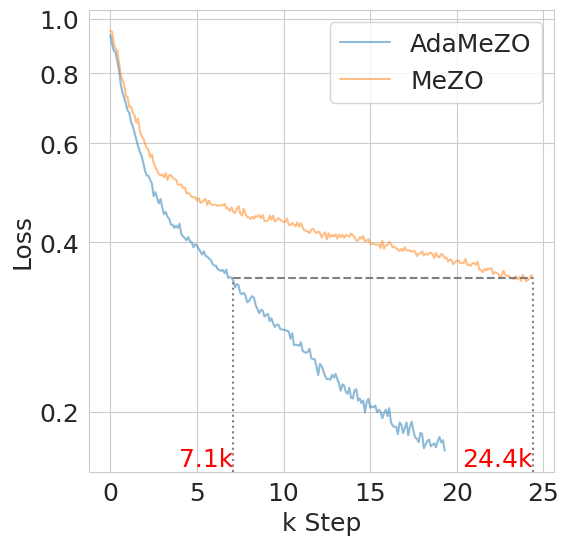}
        \subcaption{LLaMA-3b SST-2.}
        \label{fig:loss_comparison3}
    \end{minipage}
    \caption{Loss curves of MeZO and AdaMeZO on the SST2 task. When fine-tuning RoBERTa-large, OPT-1.3b, LLaMA-3b, AdaMeZO took $69.75 \%, 70.48\%,  70.90\%$ fewer forward passes to reach the loss values of MeZO at terminations, respectively. Hyperparameters and terminal conditions are detailed in \Cref{setting_fig1}.}
    \label{fig:loss_comparison}
\end{figure*}

However, adaptive optimizers retain historical gradient information in the memory. In the case of Adam, the first and second moments are the accumulated gradients and the quadratic gradients. In other words, two vectors of the same size as the model need to be kept in memory. Considering that first-order methods use backpropagation, the additional memory cost is relatively small. But in the context of zeroth-order optimizers, the memory cost is multiplied.

Adaptive zeroth-order optimizers for LLM fine-tuning have recently attracted research interest, as shown in \Cref{table:memory}. Pioneering works include HiZOO \cite{zhao2024second}, ZO-AdaMU \cite{jiang2024zo}, and Helene \cite{zhao2024helene}. HiZOO proposes approximating the diagonal Hessian using an additional forward-pass oracle, which doubles the memory required to store it. Helene is a more direct integration of zeroth-order gradient estimation and an Adam optimizer, and ZO-AdaMU replaces the moments with an uncertain version. As a result, the memory requirement is tripled to store both diagonal Hessian estimation and cumulative history gradients. However, despite the substantial increase in memory cost, they still use much less memory than first-order approaches and achieve a noticeable performance gain over MeZO.

In light of the above, we introduce AdaMeZO\footnote{Codes are available at \url{https://anonymous.4open.science/r/AdaMeZO-4547/}.}, a zeroth-order optimizer that leverages Adam-style first- and second-moment estimates to accelerate convergence without requiring additional memory to store them. This is made possible by 1) computing truncated moments that discard outdated gradients rather than faithfully maintaining the full moment estimations, and 2) block-wise generation of random gradient direction with a finer operation of the PRNG. As a result, AdaMeZO significantly reduces the number of forward passes required for convergence and improves the fine-tuned model's performance. A summary of the contributions of this work is as follows.

\begin{enumerate}
    \item We introduce AdaMeZO, an optimizer that uses zeroth-order gradient estimates and updates with Adam-style first- and second-moment estimates. Although the moments are necessary to compute the model updates, with truncated approximations and finer PRNG operations, they do not need to be stored in memory. In this way, AdaMeZO can theoretically use no additional memory to improve convergence with preconditioning.
    \item We establish a convergence bound of AdaMeZO under a non-convex assumption that recovers the convergence rate of preconditioned MeZO with multiples of memory cost. 
    \item We conduct extensive experiments to evaluate AdaMeZO's performance of AdaMeZO. We first employ 2-dimensional toy functions and visualize the trajectories of optimization. They demonstrate that AdaMeZO converges to optimal points, whereas MeZO does not, even with the same step budgets. Then we demonstrate AdaMeZO's performance by fine-tuning different models (RoBERTa \cite{liu2019roberta}, OPT \cite{zhang2022opt}, and LLaMa \cite{touvron2023llama}) for a task set identical to MeZO's. It is found that AdaMeZO almost always reaches the same termination condition as MeZO, with up to $70\%$ fewer forward passes and higher performance.
\end{enumerate}

\section{Related Works}

\subsection{Zeroth-order Optimizers for LLMs}

Zeroth-order optimization is also known as derivative-free or black-box optimization. Previously, it was used for situations where the objective function has no derivatives or obtaining derivatives is expensive. \textcolor{black}{Fine-tuning LLMs falls into the latter case and sometimes both for non-differentiable objectives \cite{tang2023zeroth, zhang2024revisiting}.} In the context of modern deep learning, it translates to the emission of auto-differentiation by backward propagation \cite{rumelhart1986learning}, resulting in hugely reduced memory consumption. Some past work on zeroth-order optimizers include \cite{spall1992multivariate, spall1997one, vakhitov2009algorithm, agarwal2009information, raginsky2011information, jamieson2012query, wang2020zeroth, baines2021fairscale}. MeZO \cite{malladi2023fine} firstly adopts the classical SPSA \cite{spall1992multivariate} to fine-tune billion-level dimension LLMs based on low rank assumptions on LLM fine-tuning, achieving comparable performance with much fewer GPU hours. A survey on concurrent extensions on top of MeZO can be found in \Cref{sec:MeZOplus}. \textcolor{black}{Notably, from-scratch zeroth-order optimization on smaller networks is also of great research interest \cite{chen2023deepzero}.}

\subsection{First-order Optimizers for LLMs}

First-order optimization algorithms form the backbone of training or fine-tuning LLMs, offering computational efficiency and scalability across billions of parameters. One of the most classic solutions is Adam \cite{kingma2014adam}, which updates based on first- and second-order moments. Some of its variants are as follows. AdamW \cite{loshchilov2017decoupled} introduces adaptive learning rates via moment estimates, achieving faster convergence on nonconvex objectives. LAMB \cite{you2019large} features a layer-wise adaptation strategy to accelerate the training of large models employing large batches. Adafactor \cite{shazeer2018adafactor} reduces memory usage by maintaining factored second‐moment estimates rather than the faithful estimates. AdaBelief \cite{zhuang2020adabelief} replaces Adam’s second moment estimates with a squared gradient with the squared difference between the gradient and its running mean to improve convergence and generalization. Lion \cite{chen2022evolved} uses only sign‐based moment updates without per-parameter scaling to reduce memory costs. Adabound \cite{luo2019adaptive} stabilizes learning rates between dynamic lower and upper thresholds to transition from adaptive behavior to SGD-like stability. RAdam \cite{liu2019variance} introduces rectification to stabilize adaptive learning rates, improving training stability in the early iterations. \cite{defazio2024road} introduced a schedule-free optimizer that requires no additional hyperparameters beyond those of standard optimizers with momentum. Interestingly, \cite{zhang2024adam} finds that block structures of diagonal Hessians can help reduce memory costs without harming performance. All of these variants feature empirical estimations of first and second moments, but with changes such as moment centering and regularization, which implies that the proposed method can also be applied to the zeroth-order versions of these variants. \textcolor{black}{Additionally, \cite{pethick2025training} explores leveraging linear minimization oracles to adapt to loss landscapes without Adam-style updates. }

\subsection{Acceleration by Adam}

Compared with first-order optimizers, second-order informed optimizers incorporate second-order information during gradient calculation. The design of Adam mimics Newton's method with second derivatives. Specifically, the second moment can be viewed as a rough approximation to the inverse Hessian. Lines of work provide analytic or numeric support for Adam's near-diagonal Hessians estimation in deep learning. \cite{das2024towards} formalizes that diagonally-dominant Hessians make Adam mathematically faster. \cite{zhang2024transformers} finds block-diagonal Hessians in real neural networks and shows Adam outperforms SGD precisely due to this structure. Empirically, \cite{elsayed2024revisiting} measures strong diagonal dominance in MLP Hessians. \cite{gui2021laplace} demonstrates that over-parameterization further drives the Hessian toward a diagonal form. Interestingly, \cite{ghorbani2019investigation} found that the Hessian spectra of deep neural networks become stable after less than $1 \%$ training step budget. \textcolor{black}{\cite{kunstner2023noise} finds that Adam's great performance could be attributed to its similarity to sign descent with momenta.}

\section{Methods}\label{sec:methods}

In this section, we first introduce the classic forward-pass-only gradient estimator, SPSA, which is the foundation of MeZO \cite{malladi2023fine}. Then, we will explain why direct splicing of SPSA with the Adam-style update rule leads to excessive memory usage and how our technique can prevent this.

\subsection{Preliminaries}

\begin{definition}[Simultaneous Perturbation Stochastic Approximation, SPSA \cite{spall1992multivariate}]
Given a model with weight $\bs{w}_t$ at step $t$ and objective function $\mc{L}$, SPSA estimates the gradient on a batch $\mc{B}$ with perturbation scale $\mu > 0$ and random direction $\bs{z}_t$ as \begin{align}\label{eq:spsa}
    \bs{g}_t = \frac{\mc{L}(\bs{w}_{t-1} + \mu \bs{z}_t, \mc{B}_t) - \mc{L}(\bs{w}_{t-1} - \mu \bs{z}_t, \mc{B}_t)}{2 \mu} \bs{z}_t.
\end{align}
\end{definition}

Following prior works, we assume $\bs{z} \sim \mc{N}(\bs{0}, I_d)$. It can be shown that $\bs{g}_t \to \nabla\mc{L}(\bs{w}_t, \mc{B}_t)$ as $\mu \to 0$, and is treated as an unbiased gradient estimator with a sufficiently small \textcolor{black}{perturbation scale} $\mu$. With an SGD-styled update rule of $\bs{w}_t \gets \bs{w}_t - \eta \bs{g}_t$, modifying the model parameters for gradient estimation and model update can be done in-place. MeZO runs quickly on GPUs since they can spawn random gradients the size of $77.5$ B within a second\footnote{\url{https://developer.nvidia.com/curand}}. As a result, MeZO generates less information in memory during fine-tuning than backpropagation. The method is shown to yield competitive performance \cite{malladi2023fine}.

\subsection{First Moment Can Be Recovered Without Additional Memory}

Using the first moment computed by history gradients with EMA updates is a widely used technique to cancel out instantaneous gradient noise, thereby promoting convergence. Common first-order algorithms require an additional trunk of memory of size $P$ to store the current first moment $\bs{m}_t$ as follows: \begin{align}
    \bs{m}_t \gets \beta_1 \bs{m}_t + (1 - \beta_1) \bs{g}_t, \quad \bs{w}_t \gets \bs{w}_t - \eta \bs{m}_t. \notag
\end{align}

\begin{algorithm}[t]
\caption{$h$-MeZO}\label{alg:hmezo}
\begin{algorithmic}
   \STATE {\bf Input:} Initialized model parameters $\boldsymbol{w}_0 \in \mathbb{R}^d$, loss function $\mathcal{L}: \mathbb{R}^d \to \mathbb{R}$, step budget $T$, perturbation scale $\mu$, learning rate $\eta$, horizon $h$, first EMA ratio $\beta_1$
   \STATE {\bf Output:} Trained model parameters $\boldsymbol{w}_T$

   \STATE {\tt seeds, projs} $\gets$ {\tt [], []}
   \FOR {$t = 1, \dots, T$}
   \STATE Sample batch $\mc{B}_t$ and random seed $s$
   \STATE Reset the PRNG with random seed $s$, spawn $\bs{z}_t \sim \mc{N}(\bs{0}, I_d)$
   \STATE Estimate $p_t$ using \Cref{eq:spsa} \COMMENT{in-place model perturbation}
   \STATE {\tt seeds.append(}$s${\tt)}, {\tt projs.append(}$p_t${\tt)}
   \STATE $\bs{w}_t \gets \bs{w}_t$
   \FOR {$\tau = 1, \dots, h$}
   \STATE $p \gets {\tt projs}[-\tau]$, $s \gets {\tt seeds}[-\tau]$
   \STATE Reset the PRNG with random seed $s$, spawn $\bs{z} \sim \mc{N}(\bs{0}, I_d)$
   \STATE $\bs{w}_t \gets \bs{w}_{t} - \eta \beta_1^{\tau - 1} p \bs{z}$
   \ENDFOR
   \ENDFOR
\end{algorithmic}
\end{algorithm}

However, the MeZO-style in-place parameter update allows the first moment to be approximated without storing history gradients. Specifically, we unroll the recursion into independent gradients, set a hyperparameter, the horizon $h$, and discard the outdated gradients computed more than $h$ steps ago, then employ a similar in-place parameter update process as in MeZO as \Cref{eq:hmezo} and detailed in \Cref{alg:hmezo}.
\begin{align}\label{eq:hmezo}
    \bs{m}_t \approx (1 - \beta_1)( \bs{g}_t + \beta_1 \bs{g}_{t-1} + \dots + \beta_1^{t-h-1} \bs{g}_{t-h-1}).
\end{align}

\begin{remark}
The idea behind \Cref{alg:hmezo} is that the share of a history gradient $\bs{g}_{t-t'}$ decays quickly. As an example, after sufficiently long steps, the share of $\bs{g}_{t-10}$ approximates $0.9 ^ {10} / (1 / (1-0.9)) \approx 0.0387$ at $\beta_1=0.9$. It implies that the key components of the first moment are several of the most recent gradients, while the rest are relatively safe to be omitted. Supported by the PRNG as a coder of the random gradients, \Cref{alg:hmezo} can use truncated first moments without additional memory.
\end{remark}

\subsection{Second Moment Informed Updates Without Additional Memory}

We can similarly recover a truncated second moment. However, bringing them into the update will still be memory-intensive. We investigate the issue and present our solution in this subsection.

An Adam-style update rule can be expressed as follows: \begin{align}\label{eq:adamupdate}
    \bs{m}_t &\gets \beta_1 \bs{m}_t + (1 - \beta_1) \bs{g}_t, \notag \\
    \bs{v}_t &\gets \beta_2 \bs{v}_t + (1 - \beta_2) \bs{g}_t \odot \bs{g}_t, \notag \\
    \bs{w}_t &\gets \bs{w}_t - \eta \frac{\bs{m}_t}{\sqrt{\bs{v}_t + \epsilon}}. \notag
\end{align}


Unlike \Cref{eq:hmezo}, we can't decompose the updates into independent gradients. Instead, we will get a summation of preconditioned gradients as follows: \begin{align}
    \frac{\bs{m}_t}{\sqrt{\bs{v}_t + \epsilon}} \approx& (1 - \beta_1)(\frac{\bs{g}_t}{\sqrt{\bs{v}_t+\epsilon}} + \frac{\beta_1 \bs{g}_{t-1}}{\sqrt{\bs{v}_t + \epsilon}} + \dots \notag \\
    &+\frac{\beta_1^{t-h-1} \bs{g}_{t-h-1}}{\sqrt{\bs{v}_t + \epsilon}}), \notag
\end{align}
with a common conditioner $\sqrt{\bs{v}_t + \epsilon}$ that can only be recovered with multiple gradients as follows: \begin{align}
    \bs{v}_t \approx& (1 - \beta_2)(\bs{g}_t \odot \bs{g}_t + \beta_2 \bs{g}_{t-1} \odot \bs{g}_{t-1} + \dots \notag\\
    & + \beta_2^{t-h-1} \bs{g}_{t-h-1} \odot \bs{g}_{t - h - 1}). \notag
\end{align} In other words, it is impossible to perform the update by multiple PRNG overwrites of the model weights. As a result, the preconditioner should be maintained in memory before modifying the model weights, which incurs substantial memory cost, as in previous work such as \cite{zhao2024second}. We aim to remove this additional memory cost in the following subsections.


\subsubsection{Fine-scaled Random Stream Generation by State Caching}

Since saving the entire preconditioner is memory-intensive, we try to perform blockwise preconditioned updates at a lower additional memory cost, which is not available in normal use of PRNG. A formal expression of how concurrent PRNG algorithms generate random number streams is as \Cref{alg:prng}, with algorithm-specified and deterministic state update function $F$ and extractor $O$.

\begin{algorithm}[t]
\caption{PRNG}\label{alg:prng}
\begin{algorithmic}
    \STATE {\bf Input:} Seed $s$
    \STATE {\bf Output:} Random number streams $r_n$
    \STATE Initial state mapper $I$ maps the random seed $s$ to the initial state $S_0$.
    \WHILE {No stop signal}
    \STATE PRNG outputs random stream $\{r_n\}$ by the recurrence \begin{align}
        r_{n} = O(S_{n-1}), \quad S_n = F(S_{n-1}),
    \end{align}
    \ENDWHILE
\end{algorithmic}
\end{algorithm}

For prior practices including \cite{malladi2023fine, zhao2024helene, zhao2024second}, the above process guarantees identical and complete $\bs{z}_t$ for gradient estimation, and gradient updating is obtained by caching the seed $s$, so that in-place gradient estimation and parameter updating without additional memory access functions. However, caching the random state $S$ offers the possibility to \emph{jump} to a specified position in a random number stream. It allows the PRNG to faithfully continue a particular random stream from wherever it left off, making it more flexible than seed caching. Code examples for this feature can be found in \Cref{sec:call}.

An illustration of the block-wise moment approximation is shown in \Cref{fig:adamezo}. A parameter block partition $\bs{w} = \{\bs{w}^{(1)}, \bs{w}^{(2)}, \dots, \bs{w}^{(b)}\}$ is prepared at the beginning of a parameter update. We first start at the first block. Processing of the first block is the same as how prior works exploit PRNGs. However, after the first random direction block $\bs{z}_{t-t'}^{(1)}$ is spawned, AdaMeZO records $S_{n}$, the corresponding random state, for each seed. When the spawning of the first block from each of the seeds within the horizon is finished, AdaMeZO skips the initial state mapping in \Cref{alg:prng}, and loads the cached $S_{n}$ to the PRNG for the next block, so that the contiguous random stream is generated rather than starting over again from the first output of the random stream. The process loops until all blocks have finished their updates.

\begin{figure}[t]
    \centering
    \includegraphics[width=1\linewidth]{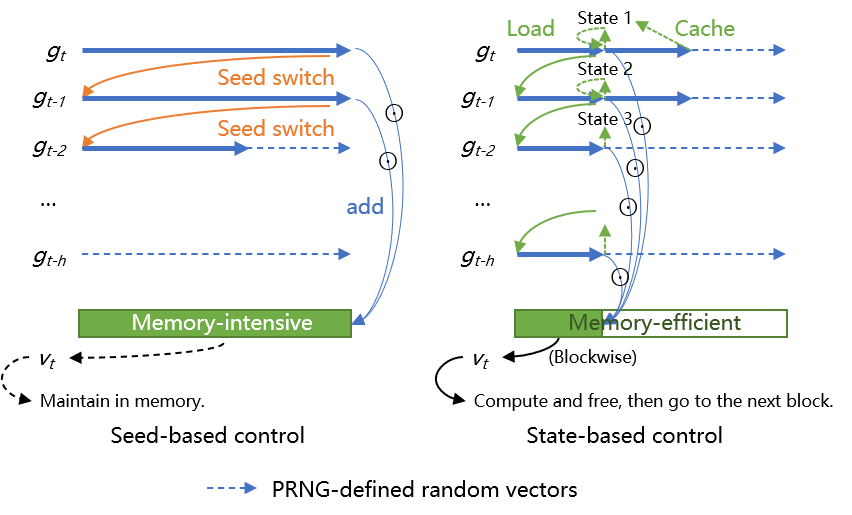}
    \caption{Block-wise moment approximation in AdaMeZO. $\odot$ denotes the Hadamard product.}
    \label{fig:adamezo}
\end{figure}


\begin{figure*}[t]
    \centering 
    \begin{subfigure}[b]{0.32\textwidth} 
        \includegraphics[width=\textwidth]{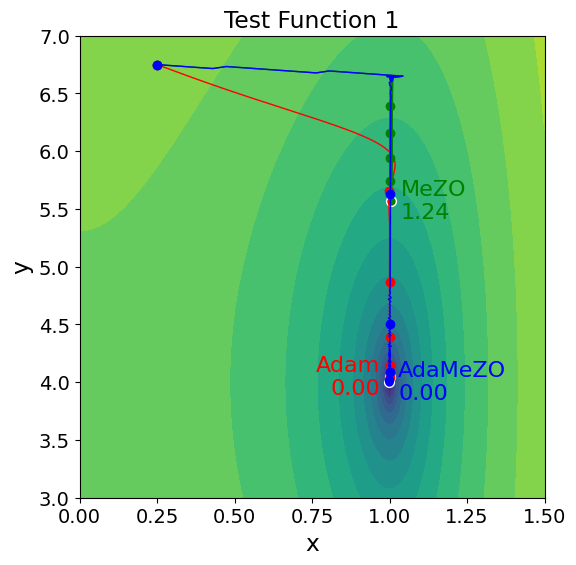} 
        \label{fig:sub1}
    \end{subfigure}
    \hfill 
    \begin{subfigure}[b]{0.32\textwidth}
        \includegraphics[width=\textwidth]{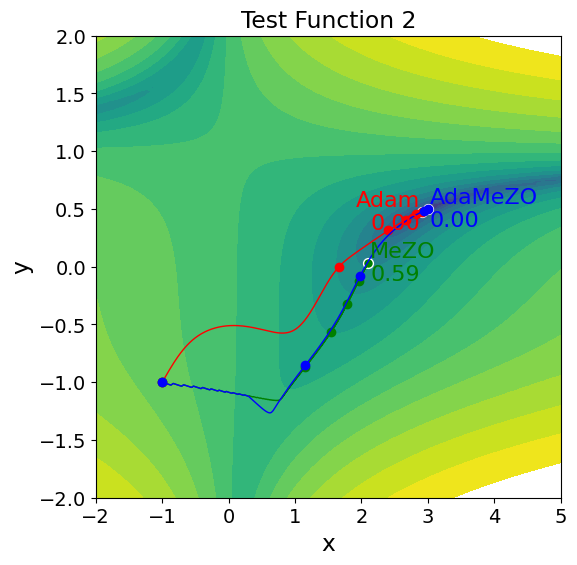}
        \label{fig:sub2}
    \end{subfigure}
    \hfill
    \begin{subfigure}[b]{0.32\textwidth}
        \includegraphics[width=\textwidth]{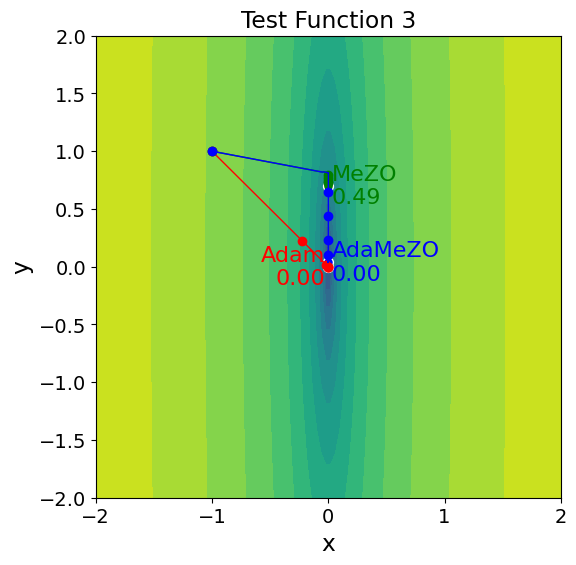}
        \label{fig:sub3}
    \end{subfigure}
    \caption{Optimization trajectories on test functions. The loss values at termination are labeled.}
    \label{fig:main}
\end{figure*}

\subsubsection{Adam-style Updates with Zeroth-Order Gradients}

With random state caching, we can update models according to \Cref{eq:adamupdate} block-wise, which is impossible for seed caching as employed by previous works, since seed caching only allows the random stream to be spawned from the initial digit. However, we need some warm-up steps to accumulate history gradients before estimating finite-horizon moments. Due to page limits, we elaborate on the process in \Cref{alg:adamezo} of the Appendix.

\begin{remark}
    It is worth mentioning that caching random states adds a minimal memory cost compared to caching seeds. In Philox \cite{salmon2011parallel}, the default choice of CUDA PRNG, random state $S$ consists of a $64$-bit random seed, a $64$-bit subsequence identifier, and a $64$-bit offset. The Mersenne Twister \cite{matsumoto1998mersenne}, the default CPU PRNG, maintains similar information to random states. Therefore, caching the random states incurs a negligible additional memory cost at the bit level compared to caching seeds.
\end{remark}

\begin{remark}
    Though the first and second moments do not go into the memory, recovering them requires a temporary additional memory trunk, whose size scales to the size of the block, corresponding to the $\delta P$ term in \Cref{table:memory}. A natural block strategy is the different layers of the model. If the model consists of $32$ layers, the additional memory introduced in the second moment approximation is roughly $2/32 P$ ($1 / 32P$ for $\bs{m}_t^{(b)}$ and $\bs{v}_t^{(b)} $ each). However, the block can be made as small as a single $ 1$-parameter block. Therefore, AdaMeZO can theoretically approximate the moments by performing frequent random state dumps and loads without incurring additional memory requirements. 
\end{remark}

\section{Theory}\label{sec:theory}

We employ the following widely adopted assumptions to facilitate an analysis.

\begin{assumption}[L-smooth]\label{ass:lsmooth}
For any weight vector $\bs{w}_1, \bs{w}_2 \in \bb{R}^d$, for a constant $0 < L < \infty$ it holds that \begin{align}
    \mc{L}(\bs{w}_2) \leq \mc{L}(\bs{w}_1) + \langle\nabla\mc{L}(\bs{w}_1), \bs{w}_2 - \bs{w}_1\rangle + \frac{L}{2} \|\bs{w}_2 - \bs{w}_1\|_2^2. \notag
\end{align}
\end{assumption}

\begin{assumption}[Bounded gradient variance]\label{ass:bgv}
The stochastic gradient $\nabla\mc{L}(\bs{w}_t, \mc{B}_t)$ has no bias and $\sigma^2$ variance due to batch stochasticity, specifically \begin{align}
    \bb{E}_{\mc{B}_t}\left[\nabla\mc{L}(\bs{w}_t, \mc{B}_t)\right] - \nabla\mc{L}(\bs{w}_t) &= 0, \\
    \bb{E}_{\mc{B}_t}\left[\left\|\nabla\mc{L}(\bs{w}_t, \mc{B}_t)\right]\right\|_2^2 - \|\nabla\mc{L}(\bs{w}_t)\|^2_2 &\leq \sigma_t^2, \quad \sigma_t < \infty. \notag
\end{align}
\end{assumption}

\begin{assumption}[Bounded second moment, \cite{zhao2024second}]\label{ass:bsm}
Each entry of $\Sigma_t$ lies in the range $[s_l, s_u]$ with $0 < s_l < s_u < \infty$. \notag
\end{assumption}

\begin{assumption}[Finite gradient drift within horizon]\label{ass:fgd}
The gradient drift within the moment horizon $h$ is finite, specifically, $\bs{m}_t$ as the first moment at step $t$ satisfies \begin{align}
    \|\bs{m}_t - \nabla \mc{L}(\bs{w}_t)\|_2 \leq \mc{O}((1-\beta_1) L \eta). \notag
\end{align}
\end{assumption}

\begin{lemma}[\cite{magnus1978moments}]\label{lm:magnus}
Let $A$ and $B$ be two symmetric matrices, $\bs{z} \sim \mc{N}(\bs{0}, \bs{I}_d)$. Define $\bs{x} = \bs{z}^\top A \bs{z} \bs{z}^\top B \bs{z}$, then it holds that \begin{align}
\bb{E}_{\bs{z}}[\bs{x}] = ({\rm tr}A)({\rm tr}B) + 2 {\rm tr}(AB). \notag
\end{align}
\end{lemma}

\begin{assumption}[Local $r$-effective rank, \cite{malladi2023fine}]\label{ass:ler}
Let $G(\bs{w}_t) := \max_{\mc{B}, |\mc{B}|=1}\|\nabla\mc{L}(\bs{w}_t, \mc{B})\|$. There is a matrix $\mc{H}(\bs{w}_t) \preceq L I_d$ satisfying: \begin{enumerate}
    \item For all $\bs{w}$ such that $\|\bs{w} - \bs{w}_t\|_2 \leq \eta d G(\bs{w}_t)$, it holds that $\nabla^2 \mc{L}(\bs{w}) \preceq H(\bs{w}_t)$.
    \item The effective rank of $H(\bs{w}_t)$, specifically, ${\rm tr}(\mc{H}(\bs{w}_t)) / \|\mc{H}(\bs{w}_t)\|_{op}$, is at most $r$.
\end{enumerate}
\end{assumption}


We present a convergence bound for non-convex optimization.

\begin{theorem}\label{thm:meanqua}
With a sufficiently small learning rate $\eta$, AdaMeZO converges to a stationary point with
    \begin{align}
        \bb{E}\left[\frac{1}{T}\sum_{t=1}^T \|\nabla \mc{L}(\bs{w}_t)\|_2^2\right] \leq \mc{O}\left(\frac{1}{\sqrt{T}}\right) + \mc{O}(\mu^2). \notag
    \end{align}
\end{theorem}

Detailed proof can be found in Appendix \ref{sec:fullproof}. The bound recovers the structure from \cite{zhao2024second}. The above result shows that after $T = \mc{O}(\epsilon^{-2})$ steps, AdaMeZO converges to a small neighborhood of a stationary point satisfying $\bb{E}\left[\|\frac{1}{T}\sum_{t=1}^T\nabla\mc{L}(\bs{w}_t)\|_2^2\right] < \epsilon$.

\begin{table*}[h]
    \setlength{\tabcolsep}{5pt}
    \caption{Results on RoBERTa-large over language tasks with $k=16$.}
    \vspace{-10pt}
    \label{table:nlp_roberta_results_k16}
    \begin{center}
    \small
        \begin{tabular}{lccccccc}
            \toprule
            \multicolumn{1}{c}{Task}  &\multicolumn{1}{c}{\bf SST-2}  &\multicolumn{1}{c}{\bf SST-5}  &\multicolumn{1}{c}{\bf SNLI}  &\multicolumn{1}{c}{\bf MNLI} &\multicolumn{1}{c}{\bf RTE}  &\multicolumn{1}{c}{\bf TREC} & \textbf{Average} \\ 
            \multicolumn{1}{c}{Type} & \multicolumn{2}{c}{---- sentiment ----} & \multicolumn{3}{c}{- natural language inference -} & \multicolumn{1}{c}{-- topic --} \\
            \hline
            Zero-shot & 79.0 & 35.5 & 50.2 & 48.8 & 51.4 & 32.0 & 49.4\\
            \hline
            
            FO \footnotesize{($\geq 4 \times$ memory)} & 91.8 & 47.5 & 77.5 & 70.0 & 66.4 & 85.0 & 73.0\\
            \hdashline
            MeZO & 90.6&	44.1&	\textbf{67.3}&	58.1&	61.6& 67.3 & 64.8 \\
             & (1.4)&	(1.0)&	(3.1)&	(1.1)&	(1.3)& (2.7) & -- \\
            MeZO-switch & 90.6&	44.3&	\textbf{67.3}&	58.0&	61.6&	67.0 & 64.8 \\
             & (1.6)&	(1.6)&	(2.8)&	(1.3)&	(2.0)&	(4.7) & -- \\
            \textit{\textbf{AdaMeZO}} & \textbf{90.9}&	\textbf{45.2}&	66.8&	\textbf{58.6}&	\textbf{63.1}& \textbf{71.5} & \textbf{66.0} \\
             & (0.9)&	(2.0)&	(2.9)&	(1.4)&	(2.3)& (5.2) & -- \\
            
            \bottomrule
        \end{tabular}
    \end{center}
\end{table*}

\begin{table*}[h]
\setlength{\tabcolsep}{2pt}
\centering
    \caption{Main results on OPT-1.3B over language tasks. Avg (w.o S,D) indicates the average metric except SQuAD and DROP.}
    \vspace{-10pt}
    \setlength{\tabcolsep}{1pt} 
    \label{table:nlp_opt_results_opt_1_3b}
    \begin{center}
    \small
        \begin{tabular}{lccccccccccccc}
            \toprule
            \multicolumn{1}{c}{Task}  &\multicolumn{1}{c}{\bf SST-2}  &\multicolumn{1}{c}{\bf RTE}  &\multicolumn{1}{c}{\bf CB}  &\multicolumn{1}{c}{\bf BoolQ} &\multicolumn{1}{c}{\bf WSC}  &\multicolumn{1}{c}{\bf WIC} &\multicolumn{1}{c}{\bf MultiRC} &\multicolumn{1}{c}{\bf COPA}  &\multicolumn{1}{c}{\bf ReCoRD} &\multicolumn{1}{c}{\bf SQuAD}  &\multicolumn{1}{c}{\bf DROP} & \textbf{Avg}  & \textcolor{black}{\textbf{Avg (w.o S,D)}} \\ 
            \multicolumn{1}{c}{Type} & \multicolumn{7}{c}{------------------ classification ------------------} & \multicolumn{2}{c}{-- multiple choice --} & \multicolumn{2}{c}{--- generation ---} \\
            \hline
            Zero-shot & 53.5&	53.4&	39.2&	45.5&	43.2&	57.5&	45.4&	75.0&	70.5&	27.2&	11.1 & 47.4 & \textcolor{black}{53.6} \\
            \hline
            FO \footnotesize{($\geq 4 \times$ memory)} & 90.9&	64.0&	77.2&	64.4&	52.8&	62.3&	65.2&	74.0&	69.1&	80.4&	28.2 & 66.2 & \textcolor{black}{68.9} \\
            & (1.2)&	(10.7)&	(7.9)&	(9.3)&	(2.0)&	(1.9)&	(6.0)&	(2.9)&	(1.2)&	(1.5)&	(1.7) & -- & \textcolor{black}{--} \\
            \hdashline
            MeZO & 90.9&	52.5&	65.5&	61.8&	51.1&	\textbf{58.6}&	53.7&	74.5&	70.6&	73.3&	22.8 & 61.4 & \textcolor{black}{64.4} \\
             & (0.3)&	(1.5)&	(6.9)&	(2.1)&	(8.4)&	(1.4)&	(2.2)&	(3.6)&	(1.0)&	(0.2)&	(0.6) & -- & \textcolor{black}{--} \\
            MeZO-switch & 91.0&	53.8&	68.7&	61.9&	52.1&	58.3&	54.9&	\textbf{75.5}&	71.0&	73.7&	24.3 & 62.3 & \textcolor{black}{65.2} \\
             & (0.6)&	(1.6)&	(2.3)&	(0.6)&	(7.6)&	(1.6)&	(1.5)&	(3.6)&	(1.2)&	(1.2)&	(1.3) & -- & \textcolor{black}{--} \\

            \color{black} HiZOO &  90.9&	\textbf{54.5}&	63.3&	62.7&	49.4&	58.4&	55.4&	74.0&	70.8&	74.5&	24.5&	61.7&	64.4  \\
             & (1.0)&	(1.6)&	(8.5)&	(1.6)&	(6.9)&	(0.4)&	(1.7)&	(1.8)&	(0.8)&	(0.4)&	(0.5)& -- & --\\
             \normalcolor
             
            \textit{\textbf{AdaMeZO}} & \textbf{91.6}&	54.3&	\textbf{69.6}&	\textbf{63.2}&	\textbf{53.5}&	58.4&	\textbf{55.9}&	\textbf{75.5}&	\textbf{71.1}&	\textbf{76.1}&	\textbf{24.6} & \textbf{63.1} & \textcolor{black}{\textbf{65.9}} \\
             & (0.3)&	(3.1)&	(1.4)&	(1.6)&	(7.8)&	(1.6)&	(0.7)&	(4.0)&	(1.3)&	(0.7)&	(1.0) & -- & \textcolor{black}{--} \\
            \bottomrule
        \end{tabular}
    \end{center}
\end{table*}

\section{Experiment Results}\label{sec:empirical}

We present empirical results for AdaMeZO with its baselines in this section. Generally, there are two types of LLMs: 1) encoder-decoder, or masked language models (MLM), such as BERT \cite{devlin2019bert} and its variants, and 2) decoder-only, or autoregressive models (ARM), such as GPT, OPT, and LLaMA families. To comprehensively demonstrate AdaMeZO's performance, we first illustrate the optimization trajectories for toy functions. Then, we test AdaMeZO with baseline algorithms on well-recognized LLMs, including an MLM RoBERTa \cite{liu2019roberta}, and two ARMs, OPT \cite{zhang2022opt} and LlaMa\cite{touvron2023llama}. 

\subsection{Toy Functions}\label{sec:toy}

It is impractical to visualize trajectories of model optimization with billions of dimensions. However, we can illustrate the optimization trajectories on three 2-dimensional toy functions as in \Cref{fig:main} to show how AdaMeZO adapts to heterogeneous curvatures. We test the Adam optimizer \cite{kingma2014adam} implemented in PyTorch \cite{paszke2019pytorch}, the vanilla MeZO \cite{malladi2023fine}, and the proposed AdaMeZO. More details in \Cref{sec:toydetail}.

In general, we observe that AdaMeZO shares Adam's curvature adaptability. Although AdaMeZO walks longer paths due to stochastic gradient directions and warm-up steps, it moves swiftly in regions of low curvature thanks to the preconditioning provided by the diagonal Hessian estimator, and the final loss values are comparable to those of Adam. In contrast, MeZO struggles with oscillations in low-curvature regions, leading to worse convergence.

\subsection{Main Results}\label{sec:mlm}

We compare the performance of AdaMeZO with vanilla MeZO and MeZO-switch, a variant of MeZO in which the learning rate is manually adjusted to ensure its optimization trajectory is longer than that of AdaMeZO. This ensures that AdaMeZO’s outperformance is not due to MeZO's underfitting, but rather to its adaptability to the loss landscape. Additionally, we include HiZOO \cite{zhao2024second} as a strong baseline that maintains preconditioners in memory. We set $h=10, \beta_1=0.7, \beta_2 = 0.9$ to evaluate performance across 4 randomly sampled data subsets of the same size and report the mean and standard deviation of the corresponding metric for each task after a hyperparameter study presented in Appendix \ref{app:hyper}.

\begin{table}[t]
\caption{Memory profile (MB) on standard PyTorch build, measured on OPT-1.3b, batch size=1. Measured via {\tt nvidia-smi}.}
\label{table:memory_profile}
\centering 
\small
\begin{tabular}{lrrrr}
\toprule
Optimizer & \textbf{SST2} & \textbf{COPA} & \textbf{SQuAD} & Memory \\ \midrule
MeZO     & 5016 & 5058 & 5040 & 1x \\ \hline
Adam & 22172 & 21660 & 22688 & 4.40x \\ 
HiZOO    & 7532 & 7535 & 7396 & 1.49x \\
\textbf{\textit{AdaMeZO}}  & \textbf{5410} & \textbf{5452} & \textbf{5434} & \textbf{1.07x} \\ 
\bottomrule
\end{tabular}
\end{table}

\begin{table*}[h]
\setlength{\tabcolsep}{2pt}
\centering
    \caption{Main results on LLaMA3-3B over language tasks.}
    \vspace{-10pt}
    \setlength{\tabcolsep}{1pt} 
    \label{table:nlp_opt_results_llama3b}
    \centering
    \begin{center}
    \small
        \begin{tabular}{lccccccccccccc}
            \toprule
            \multicolumn{1}{c}{Task}  &\multicolumn{1}{c}{\bf SST-2}  &\multicolumn{1}{c}{\bf RTE}  &\multicolumn{1}{c}{\bf CB}  &\multicolumn{1}{c}{\bf BoolQ} &\multicolumn{1}{c}{\bf WSC}  &\multicolumn{1}{c}{\bf WIC} &\multicolumn{1}{c}{\bf MultiRC} &\multicolumn{1}{c}{\bf COPA}  &\multicolumn{1}{c}{\bf ReCoRD} &\multicolumn{1}{c}{\bf SQuAD}  &\multicolumn{1}{c}{\bf DROP} & \textbf{Avg}  & \textcolor{black}{\textbf{Avg (w.o S,D)}} \\ 
            \multicolumn{1}{c}{Type} & \multicolumn{7}{c}{------------------ classification ------------------} & \multicolumn{2}{c}{-- multiple choice --} & \multicolumn{2}{c}{--- generation ---} \\
            \hline
            Zero-shot & 56.0&	52.7&	51.6&	60.9&	36.5&	54.3&	44.8&	75.0&	68.2&	47.3&	20.8 & 51.6 & \textcolor{black}{55.5} \\
            \hline
            FO \footnotesize{($\geq 4 \times$ memory)} & 92.5&	73.9&	85.6&	65.9&	57.8&	67.1&	70.6&	75.7&	68.6&	83.9&	32.2 & 70.3 & \textcolor{black}{73.1} \\
             & (0.7)&	(5.4)&	(6.9)&	(7.3)&	(7.6)&	(0.7)&	(1.8)&	(2.6)&	(1.0)&	(0.3)&	(1.8) & -- & \textcolor{black}{--} \\
            \hdashline
            MeZO & 84.5&	53.2&	64.7&	62.6&	50.4&	54.6&	52.6&	77.2&	70.0&	79.2&	26.8 & 61.4 & \textcolor{black}{63.3} \\
             & (4.9)&	(0.7)&	(2.6)&	(0.7)&	(11.3)&	(0.3)&	(2.5)&	(2.0)&	(0.4)&	(0.9)&	(0.5) & -- & \textcolor{black}{--} \\
            MeZO-switch & 86.6&	54.1&	65.5&	63.2&	51.6&	54.7&	54.7&	78.7&	70.4&	\textbf{80.4}&	27.6 & 62.5 & \textcolor{black}{64.4} \\
             & (4.5)&	(1.5)&	(0.9)&	(0.3)&	(12.2)&	(1.0)&	(0.6)&	(2.2)&	(0.6)&	(0.9)&	(0.6) & -- & \textcolor{black}{--} \\

            \color{black} HiZOO &  92.2&	54.1&	65.1&	63.7&	52.8&	\textbf{54.9}&	56.5&	\textbf{82.5}&	\textbf{71.5}&	18.5&	6.1&	56.1&	65.9  \\
             & (0.5)&	(0.3)&	(2.2)&	(0.3)&	(5.4)&	(1.7)&	(0.7)&	(0.5)&	(0.6)&	(1.9)&	(1.2) & -- & --\\
             \normalcolor
             
            \textit{\textbf{AdaMeZO}} & \textbf{92.6}&	\textbf{54.4}&	\textbf{66.0}&	\textbf{64.6}&	\textbf{54.5}&	\textbf{54.9}&	\textbf{56.9}&	81.2&	71.3&	\textbf{80.4}&	\textbf{28.1} & \textbf{64.1} & \textcolor{black}{\textbf{66.3}} \\
             & (0.5)&	(1.5)&	(1.4)&	(2.6)&	(7.5)&	(1.6)&	(1.0)&	(3.2)&	(0.9)&	(1.8)&	(1.1) & -- & \textcolor{black}{--} \\
            \bottomrule
        \end{tabular}
    \end{center}
\end{table*}

Consistent with previous research \cite{malladi2023fine}, we conduct experiments on RoBERTa-large 350M on three types of NLP tasks: sentiment, natural language inference, and topic. We sample $k=16$ examples per class to demonstrate training performance in a few-shot scenario (see Table~\ref{table:nlp_roberta_results_k16}). It is found that:


{\bf AdaMeZO yields better performance.} Averaged across all tasks, AdaMeZO achieves a \textbf{1.2\%} absolute accuracy improvement to MeZO on average, with particularly strong gains in tasks like RTE (\textbf{1.5\%}), TREC (\textbf{4.2\%}).



Then we extend our investigation to two autoregressive architectures: OPT (Table~\ref{table:nlp_opt_results_opt_1_3b}) and LLaMA3 (Table~\ref{table:nlp_opt_results_llama3b}). Experimental results show that:

{\bf AdaMeZO's superior performance scales up to billion-level LLMs.} On OPT-1.3B, AdaMeZO surpasses MeZO and MeZO-switch in all but one task. AdaMeZO achieves a \textbf{1.7\%} absolute accuracy improvement to MeZO on average, with particularly strong gains in tasks like CB (\textbf{4.1\%}), SQuAD (\textbf{2.8\%}), WSC (\textbf{2.4\%}). AdaMeZO also outperforms HiZOO with an average lead of \textbf{1.5\%}. For LLaMA-3B, AdaMeZO further extends its lead, achieves a \textbf{2.7\%} absolute accuracy improvement to MeZO on average, with particularly strong gains in tasks like SST2 (\textbf{8.1\%}), MultiRC (\textbf{4.3\%}), WSC (\textbf{4.1\%}), COPA (\textbf{4.0\%}). AdaMeZO demonstrates its scalability and maintains its advantage in modern-scale models like OPT-30B, as reported in \Cref{table:nlp_opt_results_opt30b}. We defer results of larger scale to Appendix \ref{sec:larger}.

\begin{table}[t]
\setlength{\tabcolsep}{2pt}
\centering
    \caption{\textcolor{black}{Main results on OPT-30B over language tasks, with prefix-tuning.}}
    \vspace{-10pt}
    \setlength{\tabcolsep}{6pt} 
    \fontsize{10pt}{12pt}\selectfont 
    \label{table:nlp_opt_results_opt30b}
    \centering
    \begin{center}
    \small
        \color{black}
        \begin{tabular}{lccccc}
            \toprule
            \multicolumn{1}{c}{Task}  &\multicolumn{1}{c}{\bf SST-2}  &\multicolumn{1}{c}{\bf WSC}  &\multicolumn{1}{c}{\bf WIC} &\multicolumn{1}{c}{\bf COPA} & \textbf{Avg} \\ 
            \hline
            Zero-shot & 56.6& 38.4&	50.1& 81.0& 56.5 \\
            \hline
            HiZOO & 90.1 &56.9&	55.2& 86.2& 72.1 \\
             & (1.1) &(6.2)&	(3.6) & (1.7) &-- \\
            \textit{\textbf{AdaMeZO}} & \textbf{91.1} & \textbf{57.6}&	\textbf{57.3}& \textbf{87.0} & \textbf{73.2}\\
             & (0.6) &(2.4)&	(1.5)&(1.4)& -- \\
            \bottomrule
        \end{tabular}
        \normalcolor
    \end{center}
\end{table}

\subsection{Memory Efficiency}

AdaMeZO incurs a small additional memory due to block-wise moment caching as reported in \Cref{table:memory_profile}. We can observe that, compared to optimizers that maintain actual moments, the additional memory cost is significantly reduced.

\subsection{Wall-clock Time Analysis}

AdaMeZO incurs longer per-step runtime compared to MeZO, mainly due to a) the additional PRNG calls for past gradient regeneration, and b) the weighted gradient accumulation for moment recovery. We report a runtime profile as \Cref{table:runtime_profile}. We observe that the main contributor to AdaMeZO's additional runtime is the accumulation of past gradients that are regenerated. Optimizing this accumulation process or using prefix tuning will narrow the speed gap relative to MeZO.



\begin{table}[t]
\caption{Runtime profile (sec/step) on standard PyTorch build, measured on OPT-1.3b, batch size=1.}
\label{table:runtime_profile}
\centering
\small
\begin{tabular}{@{}lrrr@{}}
\toprule
Optimizer                                               & \textbf{SST2} & \textbf{COPA} & \textbf{SQuAD} \\ \midrule
MeZO                                                             & 0.21 & 0.18 & 0.21  \\
HiZOO                                                            & 0.23 & 0.24 & 0.25  \\
MeZO + a.                                                        & 0.23 & 0.23 & 0.24  \\
\begin{tabular}[c]{@{}l@{}}MeZO + a. + b. (AdaMeZO)\end{tabular} & 0.31 & 0.30 & 0.31  \\ 
\textcolor{black}{Adam} & \textcolor{black}{0.12} & \textcolor{black}{0.13} & \textcolor{black}{0.13} \\ 
\bottomrule
\end{tabular}
\end{table}


\section{Conclusion, Limitations and Future Works}\label{sec:limitations}

In this work, we introduce AdaMeZO, the first ZO optimizer that incorporates Adam-style first- and second-moment updates without doubling or tripling the memory requirements of the original MeZO. This is achieved by estimating truncated moments and performing more refined operations on PRNGs. We provide theoretical analysis and empirical evaluations. Visualizations show that AdaMeZO adapts to complex loss landscapes without consuming excessive additional memory. Experiments on well-recognized models show that AdaMeZO reaches on-par performance using fewer forward passes and can continue to lower loss values before reaching identical terminal conditions. The paper's limitations are as follows. 

\textcolor{black}{We have captured the gradient drift across different steps using a big $\mathcal{O}$ constant related to L-smoothness, EMA weight $\beta_1$, and learning rate.} Although finite moment horizons may help to keep the estimations less biased, we did not attempt to explicitly capture the gap, which is a future research direction. 

AdaMeZO estimates second moments at a small cost, but they are inaccurate. The reason is two-fold: 1) AdaMeZO runs on zeroth-order gradient estimations, and 2) a smaller $\beta_2$ to guarantee that the discarded part contributes only a small share. Future investigations into more accurate second-moment estimations could improve performance.

\section*{Impact Statement}

This paper presents work whose goal is to advance the field of Machine
Learning. There are many potential societal consequences of our work, none
which we feel must be specifically highlighted here.

\bibliography{icml2026_conference}

@article{zhao2024helene,
  title={HELENE: Hessian Layer-wise Clipping and Gradient Annealing for Accelerating Fine-tuning LLM with Zeroth-order Optimization},
  author={Zhao, Huaqin and Li, Jiaxi and Pan, Yi and Liang, Shizhe and Yang, Xiaofeng and Liu, Wei and Li, Xiang and Dou, Fei and Liu, Tianming and Lu, Jin},
  journal={arXiv preprint arXiv:2411.10696},
  year={2024}
}

@article{kingma2014adam,
  title={Adam: A method for stochastic optimization},
  author={Kingma, Diederik P and Ba, Jimmy},
  journal={arXiv preprint arXiv:1412.6980},
  year={2014}
}

@article{malladi2023fine,
  title={Fine-tuning language models with just forward passes},
  author={Malladi, Sadhika and Gao, Tianyu and Nichani, Eshaan and Damian, Alex and Lee, Jason D and Chen, Danqi and Arora, Sanjeev},
  journal={Advances in Neural Information Processing Systems},
  volume={36},
  pages={53038--53075},
  year={2023}
}

@article{matsumoto1998mersenne,
  title={Mersenne twister: a 623-dimensionally equidistributed uniform pseudo-random number generator},
  author={Matsumoto, Makoto and Nishimura, Takuji},
  journal={ACM Transactions on Modeling and Computer Simulation (TOMACS)},
  volume={8},
  number={1},
  pages={3--30},
  year={1998},
  publisher={ACM New York, NY, USA}
}

@inproceedings{salmon2011parallel,
  title={Parallel random numbers: as easy as 1, 2, 3},
  author={Salmon, John K and Moraes, Mark A and Dror, Ron O and Shaw, David E},
  booktitle={Proceedings of 2011 international conference for high performance computing, networking, storage and analysis},
  pages={1--12},
  year={2011}
}

@article{zhao2024second,
  title={Second-order fine-tuning without pain for llms: A hessian informed zeroth-order optimizer},
  author={Zhao, Yanjun and Dang, Sizhe and Ye, Haishan and Dai, Guang and Qian, Yi and Tsang, Ivor W},
  journal={arXiv preprint arXiv:2402.15173},
  year={2024}
}

@article{spall1992multivariate,
  title={Multivariate stochastic approximation using a simultaneous perturbation gradient approximation},
  author={Spall, James C},
  journal={IEEE transactions on automatic control},
  volume={37},
  number={3},
  pages={332--341},
  year={1992},
  publisher={IEEE}
}

@article{das2024towards,
  title={Towards quantifying the preconditioning effect of adam},
  author={Das, Rudrajit and Agarwal, Naman and Sanghavi, Sujay and Dhillon, Inderjit S},
  journal={arXiv preprint arXiv:2402.07114},
  year={2024}
}

@inproceedings{ghorbani2019investigation,
  title={An investigation into neural net optimization via hessian eigenvalue density},
  author={Ghorbani, Behrooz and Krishnan, Shankar and Xiao, Ying},
  booktitle={International Conference on Machine Learning},
  pages={2232--2241},
  year={2019},
  organization={PMLR}
}

@article{liu2019roberta,
  title={Roberta: A robustly optimized bert pretraining approach},
  author={Liu, Yinhan and Ott, Myle and Goyal, Naman and Du, Jingfei and Joshi, Mandar and Chen, Danqi and Levy, Omer and Lewis, Mike and Zettlemoyer, Luke and Stoyanov, Veselin},
  journal={arXiv preprint arXiv:1907.11692},
  year={2019}
}

@article{zhang2022opt,
  title={Opt: Open pre-trained transformer language models},
  author={Zhang, Susan and Roller, Stephen and Goyal, Naman and Artetxe, Mikel and Chen, Moya and Chen, Shuohui and Dewan, Christopher and Diab, Mona and Li, Xian and Lin, Xi Victoria and others},
  journal={arXiv preprint arXiv:2205.01068},
  year={2022}
}

@article{touvron2023llama,
  title={Llama: Open and efficient foundation language models},
  author={Touvron, Hugo and Lavril, Thibaut and Izacard, Gautier and Martinet, Xavier and Lachaux, Marie-Anne and Lacroix, Timoth{\'e}e and Rozi{\`e}re, Baptiste and Goyal, Naman and Hambro, Eric and Azhar, Faisal and others},
  journal={arXiv preprint arXiv:2302.13971},
  year={2023}
}

@article{paszke2019pytorch,
  title={Pytorch: An imperative style, high-performance deep learning library},
  author={Paszke, A},
  journal={arXiv preprint arXiv:1912.01703},
  year={2019}
}

@book{beale1958iterative,
  title={On an iterative method for finding a local minimum of a function of more than one variable},
  author={Beale, Evelyn Martin Lansdowne},
  number={25},
  year={1958},
  publisher={Statistical Techniques Research Group, Section of Mathematical Statistics, Department of Mathematics, Princeton University}
}

@article{hu2022lora,
  title={Lora: Low-rank adaptation of large language models.},
  author={Hu, Edward J and Shen, Yelong and Wallis, Phillip and Allen-Zhu, Zeyuan and Li, Yuanzhi and Wang, Shean and Wang, Lu and Chen, Weizhu and others},
  journal={ICLR},
  volume={1},
  number={2},
  pages={3},
  year={2022}
}

@article{li2021prefix,
  title={Prefix-tuning: Optimizing continuous prompts for generation},
  author={Li, Xiang Lisa and Liang, Percy},
  journal={arXiv preprint arXiv:2101.00190},
  year={2021}
}

@article{lester2021power,
  title={The power of scale for parameter-efficient prompt tuning},
  author={Lester, Brian and Al-Rfou, Rami and Constant, Noah},
  journal={arXiv preprint arXiv:2104.08691},
  year={2021}
}

@article{dettmers2023qlora,
  title={Qlora: Efficient finetuning of quantized llms, 2023},
  author={Dettmers, Tim and Pagnoni, Artidoro and Holtzman, Ari and Zettlemoyer, Luke},
  journal={URL https://arxiv. org/abs/2305.14314},
  volume={2},
  year={2023}
}

@article{pan2024lisa,
  title={LISA: layerwise importance sampling for memory-efficient large language model fine-tuning},
  author={Pan, Rui and Liu, Xiang and Diao, Shizhe and Pi, Renjie and Zhang, Jipeng and Han, Chi and Zhang, Tong},
  journal={Advances in Neural Information Processing Systems},
  volume={37},
  pages={57018--57049},
  year={2024}
}

@article{rumelhart1986learning,
  title={Learning representations by back-propagating errors},
  author={Rumelhart, David E and Hinton, Geoffrey E and Williams, Ronald J},
  journal={nature},
  volume={323},
  number={6088},
  pages={533--536},
  year={1986},
  publisher={Nature Publishing Group UK London}
}

@article{bottou2018optimization,
  title={Optimization methods for large-scale machine learning},
  author={Bottou, L{\'e}on and Curtis, Frank E and Nocedal, Jorge},
  journal={SIAM review},
  volume={60},
  number={2},
  pages={223--311},
  year={2018},
  publisher={SIAM}
}

@article{loshchilov2017decoupled,
  title={Decoupled weight decay regularization},
  author={Loshchilov, Ilya and Hutter, Frank},
  journal={arXiv preprint arXiv:1711.05101},
  year={2017}
}

@article{sagun2016eigenvalues,
  title={Eigenvalues of the hessian in deep learning: Singularity and beyond},
  author={Sagun, Levent and Bottou, Leon and LeCun, Yann},
  journal={arXiv preprint arXiv:1611.07476},
  year={2016}
}

@inproceedings{zhang2023eva,
  title={Eva: Practical second-order optimization with kronecker-vectorized approximation},
  author={Zhang, Lin and Shi, Shaohuai and Li, Bo},
  booktitle={The Eleventh International Conference on Learning Representations},
  year={2023}
}

@article{yumemory,
  title={Memory-Efficient Block Coordinate Descent for Hessian-Informed Zeroth-Order Optimizer},
  author={Yu, Zhiyuan and Cheng, Yifei and Ding, Liang and Tian, Xinmei and Shen, Li and Tao, Dacheng}
}

@inproceedings{jiang2024zo,
  title={Zo-adamu optimizer: Adapting perturbation by the momentum and uncertainty in zeroth-order optimization},
  author={Jiang, Shuoran and Chen, Qingcai and Pan, Youcheng and Xiang, Yang and Lin, Yukang and Wu, Xiangping and Liu, Chuanyi and Song, Xiaobao},
  booktitle={Proceedings of the AAAI Conference on Artificial Intelligence},
  volume={38},
  number={16},
  pages={18363--18371},
  year={2024}
}

@article{you2019large,
  title={Large batch optimization for deep learning: Training bert in 76 minutes},
  author={You, Yang and Li, Jing and Reddi, Sashank and Hseu, Jonathan and Kumar, Sanjiv and Bhojanapalli, Srinadh and Song, Xiaodan and Demmel, James and Keutzer, Kurt and Hsieh, Cho-Jui},
  journal={arXiv preprint arXiv:1904.00962},
  year={2019}
}

@inproceedings{shazeer2018adafactor,
  title={Adafactor: Adaptive learning rates with sublinear memory cost},
  author={Shazeer, Noam and Stern, Mitchell},
  booktitle={International Conference on Machine Learning},
  pages={4596--4604},
  year={2018},
  organization={PMLR}
}

@article{zhuang2020adabelief,
  title={Adabelief optimizer: Adapting stepsizes by the belief in observed gradients},
  author={Zhuang, Juntang and Tang, Tommy and Ding, Yifan and Tatikonda, Sekhar C and Dvornek, Nicha and Papademetris, Xenophon and Duncan, James},
  journal={Advances in neural information processing systems},
  volume={33},
  pages={18795--18806},
  year={2020}
}

@inproceedings{chen2022evolved,
  title={Evolved optimizer for vision},
  author={Chen, Xiangning and Liang, Chen and Huang, Da and Real, Esteban and Liu, Yao and Wang, Kaiyuan and Hsieh, Cho-Jui and Lu, Yifeng and Le, Quoc V},
  booktitle={First Conference on Automated Machine Learning (Late-Breaking Workshop)},
  year={2022}
}

@article{luo2019adaptive,
  title={Adaptive gradient methods with dynamic bound of learning rate},
  author={Luo, Liangchen and Xiong, Yuanhao and Liu, Yan and Sun, Xu},
  journal={arXiv preprint arXiv:1902.09843},
  year={2019}
}

@article{liu2019variance,
  title={On the variance of the adaptive learning rate and beyond},
  author={Liu, Liyuan and Jiang, Haoming and He, Pengcheng and Chen, Weizhu and Liu, Xiaodong and Gao, Jianfeng and Han, Jiawei},
  journal={arXiv preprint arXiv:1908.03265},
  year={2019}
}

@article{zhang2024transformers,
  title={Why transformers need adam: A hessian perspective},
  author={Zhang, Yushun and Chen, Congliang and Ding, Tian and Li, Ziniu and Sun, Ruoyu and Luo, Zhiquan},
  journal={Advances in Neural Information Processing Systems},
  volume={37},
  pages={131786--131823},
  year={2024}
}

@article{elsayed2024revisiting,
  title={Revisiting scalable hessian diagonal approximations for applications in reinforcement learning},
  author={Elsayed, Mohamed and Farrahi, Homayoon and Dangel, Felix and Mahmood, A Rupam},
  journal={arXiv preprint arXiv:2406.03276},
  year={2024}
}

@inproceedings{gui2021laplace,
  title={Laplace ap-proximation with diagonalized hessian for over-parameterized neural networks},
  author={Gui, Ming and Zhao, Ziqing and Qiu, Tianming and Shen, Hao},
  booktitle={NeurIPS Workshop on Bayesian Deep Learning},
  year={2021}
}

@misc{baines2021fairscale,
  title={Fairscale: A general purpose modular pytorch library for high performance and large scale training},
  author={Baines, Mandeep and Bhosale, Shruti and Caggiano, Vittorio and Goyal, Naman and Goyal, Siddharth and Ott, Myle and Lefaudeux, Benjamin and Liptchinsky, Vitaliy and Rabbat, Mike and Sheiffer, Sam and others},
  year={2021}
}

@article{vakhitov2009algorithm,
  title={Algorithm for stochastic approximation with trial input perturbation in the nonstationary problem of optimization},
  author={Vakhitov, Alexander Timurovich and Granichin, Oleg Nikolaevich and Gurevich, Lev Stanislavovich},
  journal={Automation and Remote Control},
  volume={70},
  pages={1827--1835},
  year={2009},
  publisher={Springer}
}

@article{spall1997one,
  title={A one-measurement form of simultaneous perturbation stochastic approximation},
  author={Spall, James C},
  journal={Automatica},
  volume={33},
  number={1},
  pages={109--112},
  year={1997},
  publisher={Elsevier}
}

@article{jamieson2012query,
  title={Query complexity of derivative-free optimization},
  author={Jamieson, Kevin G and Nowak, Robert and Recht, Ben},
  journal={Advances in Neural Information Processing Systems},
  volume={25},
  year={2012}
}

@article{agarwal2009information,
  title={Information-theoretic lower bounds on the oracle complexity of convex optimization},
  author={Agarwal, Alekh and Wainwright, Martin J and Bartlett, Peter and Ravikumar, Pradeep},
  journal={Advances in Neural Information Processing Systems},
  volume={22},
  year={2009}
}

@article{raginsky2011information,
  title={Information-based complexity, feedback and dynamics in convex programming},
  author={Raginsky, Maxim and Rakhlin, Alexander},
  journal={IEEE Transactions on Information Theory},
  volume={57},
  number={10},
  pages={7036--7056},
  year={2011},
  publisher={IEEE}
}

@article{wang2020zeroth,
  title={Zeroth-order algorithms for nonconvex minimax problems with improved complexities},
  author={Wang, Zhongruo and Balasubramanian, Krishnakumar and Ma, Shiqian and Razaviyayn, Meisam},
  journal={arXiv preprint arXiv:2001.07819},
  year={2020}
}

@article{liu2024sparse,
  title={Sparse mezo: Less parameters for better performance in zeroth-order llm fine-tuning},
  author={Liu, Yong and Zhu, Zirui and Gong, Chaoyu and Cheng, Minhao and Hsieh, Cho-Jui and You, Yang},
  journal={arXiv preprint arXiv:2402.15751},
  year={2024}
}

@article{guo2024zeroth,
  title={Zeroth-order fine-tuning of llms with extreme sparsity},
  author={Guo, Wentao and Long, Jikai and Zeng, Yimeng and Liu, Zirui and Yang, Xinyu and Ran, Yide and Gardner, Jacob R and Bastani, Osbert and De Sa, Christopher and Yu, Xiaodong and others},
  journal={arXiv preprint arXiv:2406.02913},
  year={2024}
}

@article{chen2024enhancing,
  title={Enhancing zeroth-order fine-tuning for language models with low-rank structures},
  author={Chen, Yiming and Zhang, Yuan and Cao, Liyuan and Yuan, Kun and Wen, Zaiwen},
  journal={arXiv preprint arXiv:2410.07698},
  year={2024}
}

@article{chen2025towards,
  title={Towards Efficient Low-order Hybrid Optimizer for Language Model Fine-tuning},
  author={Chen, Minping and Huang, You-Liang and Wen, Zeyi},
  year={2025}
}

@article{tan2025harmony,
  title={Harmony in Divergence: Towards Fast, Accurate, and Memory-efficient Zeroth-order LLM Fine-tuning},
  author={Tan, Qitao and Liu, Jun and Zhan, Zheng and Ding, Caiwei and Wang, Yanzhi and Lu, Jin and Yuan, Geng},
  journal={arXiv preprint arXiv:2502.03304},
  year={2025}
}

@article{chen2025memory,
  title={A Memory Efficient Randomized Subspace Optimization Method for Training Large Language Models},
  author={Chen, Yiming and Zhang, Yuan and Liu, Yin and Yuan, Kun and Wen, Zaiwen},
  journal={arXiv preprint arXiv:2502.07222},
  year={2025}
}

@article{sun2025tezo,
  title={TeZO: Empowering the Low-Rankness on the Temporal Dimension in the Zeroth-Order Optimization for Fine-tuning LLMs},
  author={Sun, Yan and Huang, Tiansheng and Ding, Liang and Shen, Li and Tao, Dacheng},
  journal={arXiv preprint arXiv:2501.19057},
  year={2025}
}

@book{magnus1978moments,
  title={The moments of products of quadratic forms in normal variables},
  author={Magnus, Jan R and others},
  year={1978},
  publisher={Univ., Instituut voor Actuariaat en Econometrie}
}

@article{zhang2022adam,
  title={Adam can converge without any modification on update rules},
  author={Zhang, Yushun and Chen, Congliang and Shi, Naichen and Sun, Ruoyu and Luo, Zhi-Quan},
  journal={Advances in neural information processing systems},
  volume={35},
  pages={28386--28399},
  year={2022}
}

@article{tang2023zeroth,
  title={Zeroth-order optimization meets human feedback: Provable learning via ranking oracles},
  author={Tang, Zhiwei and Rybin, Dmitry and Chang, Tsung-Hui},
  journal={arXiv preprint arXiv:2303.03751},
  year={2023}
}

@inproceedings{devlin2019bert,
  title={Bert: Pre-training of deep bidirectional transformers for language understanding},
  author={Devlin, Jacob and Chang, Ming-Wei and Lee, Kenton and Toutanova, Kristina},
  booktitle={Proceedings of the 2019 conference of the North American chapter of the association for computational linguistics: human language technologies, volume 1 (long and short papers)},
  pages={4171--4186},
  year={2019}
}

@article{zhang2024adam,
  title={Adam-mini: Use fewer learning rates to gain more},
  author={Zhang, Yushun and Chen, Congliang and Li, Ziniu and Ding, Tian and Wu, Chenwei and Kingma, Diederik P and Ye, Yinyu and Luo, Zhi-Quan and Sun, Ruoyu},
  journal={arXiv preprint arXiv:2406.16793},
  year={2024}
}

@article{zhang2024revisiting,
  title={Revisiting zeroth-order optimization for memory-efficient llm fine-tuning: A benchmark},
  author={Zhang, Yihua and Li, Pingzhi and Hong, Junyuan and Li, Jiaxiang and Zhang, Yimeng and Zheng, Wenqing and Chen, Pin-Yu and Lee, Jason D and Yin, Wotao and Hong, Mingyi and others},
  journal={arXiv preprint arXiv:2402.11592},
  year={2024}
}

@article{chen2023deepzero,
  title={Deepzero: Scaling up zeroth-order optimization for deep model training},
  author={Chen, Aochuan and Zhang, Yimeng and Jia, Jinghan and Diffenderfer, James and Liu, Jiancheng and Parasyris, Konstantinos and Zhang, Yihua and Zhang, Zheng and Kailkhura, Bhavya and Liu, Sijia},
  journal={arXiv preprint arXiv:2310.02025},
  year={2023}
}

@article{pethick2025training,
  title={Training deep learning models with norm-constrained lmos},
  author={Pethick, Thomas and Xie, Wanyun and Antonakopoulos, Kimon and Zhu, Zhenyu and Silveti-Falls, Antonio and Cevher, Volkan},
  journal={arXiv preprint arXiv:2502.07529},
  year={2025}
}

@article{defazio2024road,
  title={The road less scheduled},
  author={Defazio, Aaron and Yang, Xingyu and Mehta, Harsh and Mishchenko, Konstantin and Khaled, Ahmed and Cutkosky, Ashok},
  journal={Advances in Neural Information Processing Systems},
  volume={37},
  pages={9974--10007},
  year={2024}
}

@article{kunstner2023noise,
  title={Noise is not the main factor behind the gap between sgd and adam on transformers, but sign descent might be},
  author={Kunstner, Frederik and Chen, Jacques and Lavington, Jonathan Wilder and Schmidt, Mark},
  journal={arXiv preprint arXiv:2304.13960},
  year={2023}
}
\bibliographystyle{icml2026}

\newpage
\appendix
\onecolumn

\section{Additional Related Works}\label{sec:MeZOplus}


In addition to MeZO, numerous subsequent excellent works have emerged to enhance the vanilla version. \cite{jiang2024zo} incorporates uncertain moments estimations to promote convergence. \cite{zhao2024helene} invokes Adam-style update rules for better performance. \cite{zhao2024second} estimates the diagonal Hessian with a three-point second derivative estimation, admitting a third forward pass for each step. \cite{liu2024sparse, guo2024zeroth} proposed to insert sparsity for better performance. \cite{chen2024enhancing, sun2025tezo} exploits the low-rank property for better performance. \cite{chen2025towards} proposes a hybrid optimizer that balances efficiency and trade-offs. \cite{tan2025harmony} explores a layer-wise adaptation to speed up zeroth-order fine-tuning. \cite{chen2025memory} investigates memory-efficient zeroth-order fine-tuning from a subspace-optimization perspective. \cite{yumemory} introduces a block version of HiZOO, attempting to preserve preconditioning-improved convergence while reducing additional memory access. 

\section{Detailed Experiment Settings}



\subsection{Computation Resources}\label{appendix:computation_resources}

We summarize the computational devices for empirical evaluations in \Cref{table:compuational}. We use device 1 for MLM experiments and device 2 for ARM experiments. 








\begin{table}[h]
\caption{Summary of computational devices for empirical evaluations.}
\label{table:compuational}
\centering
\setlength{\tabcolsep}{2.5pt}
\small
\begin{tabular}{@{}cccccc@{}}
\toprule
Device                              & OS/CPU/GPU                                     & Python                   & PyTorch                & CUDA                  & cuDNN                \\ \midrule
\multirow{3}{*}{1}  & Linux 5.10.0, amd64                   & \multirow{3}{*}{3.10.13} & \multirow{3}{*}{2.3.0} & \multirow{3}{*}{12.1} & \multirow{3}{*}{8.9} \\
                                  & Intel(R) Xeon(R) Gold 6133 CPU @ 2.50GHz       &                          &                        &                       &                      \\
                                  & 6x NVIDIA GeForce RTX 3090 GPU                 &                          &                        &                       &                      \\ \midrule
\multirow{3}{*}{2}  & Linux 4.18.0, x86\_64                     & \multirow{3}{*}{3.9.7}  & \multirow{3}{*}{2.1.0} & \multirow{3}{*}{12.1} & \multirow{3}{*}{8.9} \\
                                  & AMD EPYC 7742 64-Core Processor                &                          &                        &                       &                      \\
                                  & 4x NVIDIA A100-SXM4-80GB                       &                          &                        &                       &   
                                  \\ \bottomrule
\end{tabular}
\end{table}

\subsection{Formal Pseudo-codes for AdaMeZO}

A formal description of AdaMeZO in pseudo-code is as \Cref{alg:adamezo}.

\subsection{Hyperparameter Study}\label{app:hyper}

There are $3$ new hyperparameters introduced in AdaMeZO, the horizon $h$, and the EMA weight in momentum update $(\beta_1, \beta_2)$ as in Adam. 

We first analyze how different $h$ impacts the evaluation loss curve, as shown in \Cref{fig:differenth}. We find that setting $h > 5$ admits comparable results, while $h<5$ requires further optimization to avoid crash due to numerical instability. We set $h=10$ in experiments for stability, and set $h=5$ for wall-clock time analysis.

Similar to Adam, two new hyperparameters $(\beta_1, \beta_2)$ are introduced into the algorithm. We report AdaMeZO's performance across different hyperparameter settings, as shown in \Cref{table:beta}, on OPT-1.3B with the SST2 task, since complete task sweeps are computationally expensive. It can be observed that the first moments improve MeZO's performance, and the second moments further improve it. The performance gain is robust against reasonable choices of the hyperparameter $(\beta_1, \beta_2)$.

\begin{figure}[t]
    \centering
    \includegraphics[width=1\linewidth]{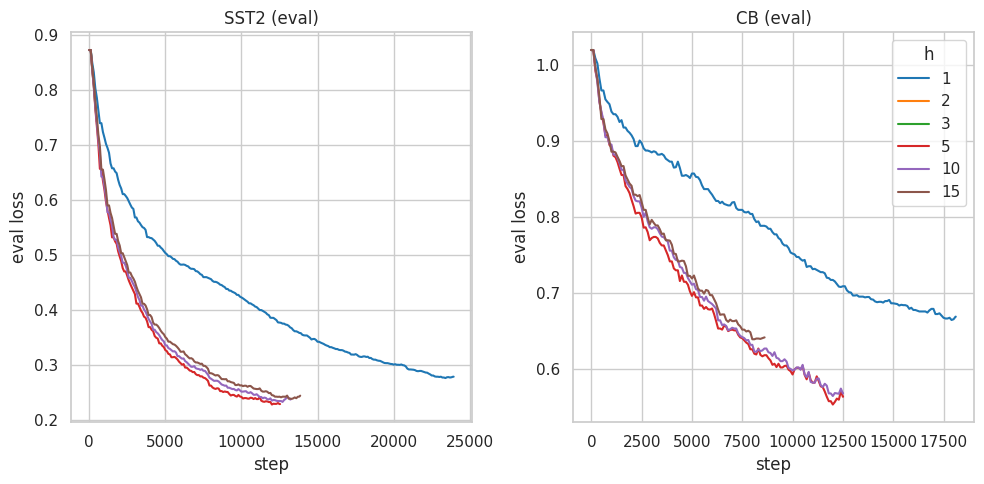}
    \caption{Evaluation loss with different $h$.}
    \label{fig:differenth}
\end{figure}

\begin{table}[h]
\caption{Performance comparison with different $(\beta_1, \beta_2)$.}
\vspace{-10pt}
\label{table:beta}
\begin{center}
\small
\begin{tabular}{lrrr}
\toprule
$(\beta_1, \beta_2)$ & SST2                              & COPA                              & SQuAD                             \\ \midrule
(0.7, 0.9)           & 91.6 (0.3)                        & 75.5 (4.0)                        & 76.1 (0.7)                        \\ \midrule
(0.7, 0.99)          & {\color[HTML]{333333} 90.9 (0.9)} & 75.3 (2.9)                        & 75.6 (0.9)                        \\
(0.6, 0.9)           & {\color[HTML]{333333} 91.1 (0.6)} & {\color[HTML]{333333} 75.8 (2.9)} & {\color[HTML]{333333} 75.6 (1.2)} \\
(0.8, 0.9)           & 91.5 (0.7)                        & {\color[HTML]{333333} 74.3 (2.9)} & 75.6 (1.5)                        \\
(0.7, 0.0), mSGD     & {\color[HTML]{333333} 90.9 (0.9)} & {\color[HTML]{333333} 75.3 (2.9)} & {\color[HTML]{333333} 75.6 (0.9)} \\
(0.0, 0.0), MeZO     & 90.9 (0.3)                        & {\color[HTML]{333333} 74.5 (3.6)} & {\color[HTML]{333333} 73.3 (0.2)} \\ \bottomrule
\end{tabular}
\end{center}
\end{table}

\subsection{Detailed Settings for \Cref{fig:loss_comparison}}\label{setting_fig1}

For (a), the learning rate is 1e-6, with $16$ training samples per class. For (b) and (c), the learning rate is 1e-7, with $1000$ training samples in total. We set $\beta_1=0.7, \beta_2=0.9, h=10$, so that AdaMeZO discards only a small part by truncating the moments and admits a smoother second moment estimation compared to the first. With an abuse of context, the choice of $(\beta_1, \beta_2)$ falls into the suggested area $0 < \beta_1 < \sqrt{\beta_2} < 1$ by \cite{zhang2022adam}. Fine-tuning terminates when either of the following happens. \begin{enumerate}
    \item Measure evaluation loss per $100$ steps. Evaluation loss does not drop for $5$ continual measures.
    \item Number of steps exceeds $40000$.
\end{enumerate}

\subsection{Detailed Settings for \Cref{sec:toy}}\label{sec:toydetail}

The expressions of the test functions are \begin{enumerate}
    \item $f_1(x, y) = 8 (x-1) ^ 2 (1.3 x ^2 + 2 x + 1) + 0.5 (y - 4) ^ 2$ \cite{zhao2024second}.
    \item $f_2(x, y) = (1.5 - x + xy)^2 + (2.25 - x + xy^2)^2 + (2.625 - x + xy^3)^2$ \cite{beale1958iterative}.
    \item $f_3(x, y) = 100 x^2 + y ^ 2$.
\end{enumerate}

Specifications on implementations are as \Cref{table:toy}. The setting follows on the following rule: \begin{enumerate}
    \item Set the learning rate of Adam to $0.01$,
    \item Tune the learning rate for ZO optimizers so that the trajectory lengths are comparable to Adam's. We allow a longer trajectory ($<1.6 \times$) for ZO optimizers.
\end{enumerate}

For MeZO and AdaMeZO, we allow only $2$ seeds coding $2$ gradient directions. This is to capture the situation where the number of steps, equivalently the total number of explored gradient directions (in thousands), is usually less than the number of dimensions of the LLMs (in billions).

\begin{table}[h]
\caption{Specifications for toy functions.}
\label{table:toy}
\centering
\small
\begin{tabular}{rclclclcc}
\toprule
\multicolumn{1}{l}{} & \multicolumn{2}{c}{Adam}            & \multicolumn{2}{c}{MeZO}             & \multicolumn{2}{c}{AdaMeZO}          & \multicolumn{1}{l}{\multirow{2}{*}{\# steps}} & \multicolumn{1}{l}{\multirow{2}{*}{Initialization}} \\ \cline{2-7}
\multicolumn{1}{l}{} & lr     & \multicolumn{1}{c}{length} & lr      & \multicolumn{1}{c}{length} & lr      & \multicolumn{1}{c}{length} & \multicolumn{1}{l}{}                          & \multicolumn{1}{l}{}                                \\ \hline
$f_1$                & $0.01$ & $3.0227$                   & $0.01$  & $4.6659$                   & $0.01$  & $4.5078$                   & $600$                                         & $(0.2, 6.75)$                                       \\
$f_2$                & $0.01$ & $4.3597$                   & $0.002$ & $5.5405$                   & $0.002$ & $5.3207$                   & $2500$                                        & $(-1, -1)$                                          \\
$f_3$                & $0.01$ & $1.4142$                   & $0.01$  & $1.4243$                   & $0.01$  & $1.8577$                   & $500$                                         & $(-1, 1)$                                           \\ 
\bottomrule
\end{tabular}
\end{table}

Trajectories at higher resolutions and 3D views of the loss landscapes are shown in \Cref{fig:bigtoy}.

\subsection{Detailed Settings for \Cref{sec:mlm}}\label{sec:mlmdetail}




\begin{table}[h]
\small
\caption{\small Hyperparameter settings.}
\centering
\begin{tabular}{ccccccc}
\toprule
\multicolumn{1}{c}{}     & $B$                   & $T$                              & $\eta$             & $q$                   & $\mu$                      & $(\beta_1, \beta_2)$ \\ \hline
\multirow{1}{*}{Table \ref{table:nlp_roberta_results_k16}}                  &           \multirow{1}{*}{$16$}            &                  \multirow{1}{*}{$1 \times 10^{5}$}                &          $1 \times 10 ^ {-6}$          &            \multirow{1}{*}{$5$}           &              \multirow{1}{*}{$1 \times 10^{-3}$}              &            \multirow{1}{*}{$(0.7, 0.9)$}                    \\
\multirow{1}{*}{Table \ref{table:nlp_opt_results_opt_1_3b}}                  &          \multirow{1}{*}{$16$}             &                     \multirow{1}{*}{$4 \times 10^{4}$}             &           $1 \times 10 ^ {-7}$         &              \multirow{1}{*}{$5$}         &              \multirow{1}{*}{$1 \times 10^{-3}$}              &              \multirow{1}{*}{$(0.7, 0.9)$}                  \\
\multirow{1}{*}{Table \ref{table:nlp_opt_results_llama3b}}                  &          \multirow{1}{*}{$16$}             &                     \multirow{1}{*}{$4 \times 10^{4}$}             &           $1 \times 10 ^ {-7}$         &              \multirow{1}{*}{$5$}         &              \multirow{1}{*}{$1 \times 10^{-3}$}              &              \multirow{1}{*}{$(0.7, 0.9)$}                  \\
\bottomrule
\end{tabular}
\end{table}

Fine-tuning terminates when either of the following conditions is met. \begin{enumerate}
    \item Measure evaluation loss per $100$ steps. Evaluation loss does not drop for $q$ continual measures.
    \item Number of steps exceeds $T$.
\end{enumerate}






\section{Additional Experiment Results}

\subsection{Larger Models}\label{sec:larger}

We report the performance of AdaMeZO on larger models to demonstrate the scalability of the optimizer as \Cref{table:nlp_opt_results_llama7b} \textcolor{black}{and \Cref{table:nlp_opt_results_opt13b}.}

\textcolor{black}{We also report the hyperparameter settings in \Cref{table:hyper_prefix} for \Cref{table:nlp_opt_results_opt30b}.}

\textcolor{black}{We include training loss and evaluation loss curves in \Cref{fig:trainloss13b} and \Cref{fig:evalloss13b}, respectively.}

\begin{table*}[h]
\setlength{\tabcolsep}{2pt}
\centering
    \caption{Main results on LLaMA-7B over language tasks.}
    \vspace{-10pt}
    \setlength{\tabcolsep}{1pt} 
    \label{table:nlp_opt_results_llama7b}
    \scriptsize
    \begin{center}
    \small
        \begin{tabular}{lccccccccccccc}
            \toprule
            \multicolumn{1}{c}{Task}  &\multicolumn{1}{c}{\bf SST-2}  &\multicolumn{1}{c}{\bf RTE}  &\multicolumn{1}{c}{\bf CB}  &\multicolumn{1}{c}{\bf BoolQ} &\multicolumn{1}{c}{\bf WSC}  &\multicolumn{1}{c}{\bf WIC} &\multicolumn{1}{c}{\bf MultiRC} &\multicolumn{1}{c}{\bf COPA}  &\multicolumn{1}{c}{\bf ReCoRD} &\multicolumn{1}{c}{\bf SQuAD}  &\multicolumn{1}{c}{\bf DROP} &  \textbf{Avg}&  \textcolor{black}{\textbf{Avg (w.o S,D)}} \\ 
            \multicolumn{1}{c}{Type} & \multicolumn{7}{c}{------------------ classification ------------------} & \multicolumn{2}{c}{-- multiple choice --} & \multicolumn{2}{c}{--- generation ---} \\
            \hline
            Zero-shot & 59.7&	49.8&	48.2&	65.0&	56.7&	50.6&	50.5&	84.0&	79.9&	58.6&	17.5&	56.4 & \textcolor{black}{60.4} \\
            \hline
            FO \footnotesize{($\geq 4\times$ memory)} & 95.0&	86.0&	94.1&	83.1&	54.5&	66.2&	79.3&	81.2&	75.4&	89.2&	39.7&	76.7 & \textcolor{black}{79.4} \\
            & (0.5)&	(2.2)&	(1.7)&	(0.5)&	(5.8)&	(4.9)&	(3.0)&	(2.2)&	(2.3)&	(1.0)&	(1.0) & -- & \textcolor{black}{--}\\
            \hdashline
            MeZO & 85.7&	54.7&	58.8&	68.3&	58.1&	56.9&	60.9&	82.5&	78.0&	71.9&	30.9&	64.2 & \textcolor{black}{67.1} \\
             & (1.9)&	(0.5)&	(3.8)&	(1.5)&	(2.9)&	(1.7)&	(2.7)&	(1.2)&	(1.8)&	(4.5)&	(1.1) & -- & \textcolor{black}{--} \\
            MeZO-switch & 87.2&	55.2&	60.6&	68.7&	60.2&	56.8&	60.5&	84.0&	80.3&	78.8&	32.3&	65.8 & \textcolor{black}{68.1} \\
             & (0.7)&	(1.2)&	(6.3)&	(1.2)&	(1.2)&	(0.5)&	(2.3)&	(0.8)&	(0.5)&	(3.2)&	(1.1) & -- & \textcolor{black}{--} \\

            \color{black} HiZOO &  90.9&	59.7&	\textbf{63.3}&	70.3&	59.8&	57.4&	\textbf{62.7}&	83.7&	79.3 & 21.3&	4.6 &59.4 &69.7 \\
             & (2.5)&	(2.9)&	(0.9)&	(1.5)&	(6.7)&	(0.2)&	(2.4)&	(1.2)&	(1.6) &(1.1) &(0.9) & -- & --\\
             \normalcolor
             
            \textit{\textbf{AdaMeZO}} & \textbf{91.4}&	\textbf{61.2}&	62.9&	\textbf{70.9}&	\textbf{60.5}&	\textbf{57.6}&	62.1&	\textbf{84.5}&	\textbf{80.5}&	\textbf{84.9}&	\textbf{36.2}&	\textbf{68.4} & \textcolor{black}{\textbf{70.2}} \\
             & (2.5)&	(2.6)&	(1.6)&	(2.2)&	(2.0)&	(1.1)&	(2.6)&	(3.1)&	(0.9)&	(0.9)&	(2.1) & -- & \textcolor{black}{--} \\
            \bottomrule
        \end{tabular}
    \end{center}
\end{table*}

\begin{table*}[h!]
\setlength{\tabcolsep}{2pt}
    \caption{Main results on OPT-13B over language tasks. OOM indicates that HiZOO encountered an out-of-memory error. A cell is marked as OOM if any of the evaluation seeds trigger an OOM. The official HiZOO implementation only supports single-GPU training, and the memory overhead of some instances exceeds the capacity of our largest GPU (A100 80GB), leading to OOM failures.}
    \vspace{-10pt}
    \setlength{\tabcolsep}{1pt} 
    \label{table:nlp_opt_results_opt13b}
    \centering
    \begin{center}
    \small
        \begin{tabular}{lccccccccccccc}
            \toprule
            \multicolumn{1}{c}{Task}  &\multicolumn{1}{c}{\bf SST-2}  &\multicolumn{1}{c}{\bf RTE}  &\multicolumn{1}{c}{\bf CB}  &\multicolumn{1}{c}{\bf BoolQ} &\multicolumn{1}{c}{\bf WSC}  &\multicolumn{1}{c}{\bf WIC} &\multicolumn{1}{c}{\bf MultiRC} &\multicolumn{1}{c}{\bf COPA}  &\multicolumn{1}{c}{\bf ReCoRD} &\multicolumn{1}{c}{\bf SQuAD}  &\multicolumn{1}{c}{\bf DROP} & \textbf{Avg}  & \textcolor{black}{\textbf{Avg (w.o S,D)}} \\ 
            \multicolumn{1}{c}{Type} & \multicolumn{7}{c}{------------------ classification ------------------} & \multicolumn{2}{c}{-- multiple choice --} & \multicolumn{2}{c}{--- generation ---} \\
            \hline
            Zero-shot & 58.8&	59.6&	46.4&	59.0&	38.5&	55.0&	46.9&	80.0&	81.2&	46.2&	14.6 & 53.2 & \textcolor{black}{58.3} \\
            \hline
            FO \footnotesize{($\geq 4 \times$ memory)} & 92.0&	70.8&	83.9&	77.1&	63.5&	55.0&	71.1&	79.0&	74.1&	84.9&	31.3 & 71.1 & \textcolor{black}{74.0} \\
            \hdashline
            MeZO & 92.1&	60.4&	67.8&	65.5&	56.6&	54.9&	56.7&	\textbf{87.0}&	80.2&	82.1&	30.6 & 66.7 & \textcolor{black}{69.0} \\
             & (0.5)&	(0.6)&	(1.4)&	(3.0)&	(7.9)&	(1.7)&	(0.8)&	(1.1)&	(1.0)&	(1.3)&	(1.5) & -- & \textcolor{black}{--} \\
            MeZO-switch & 92.6&	61.6&	66.9&	66.2&	56.9&	55.4&	57.5&	86.0&	\textbf{80.5}&	83.4&	30.5 & 67.0 & \textcolor{black}{69.3} \\
             & (0.2)&	(2.5)&	(1.0)&	(3.7)&	(8.4)&	(0.7)&	(0.5)&	(2.7)&	(1.1)&	(0.8)&	(0.9) & -- &\textcolor{black}{--} \\

            \color{black} HiZOO & 91.5&	62.5&	\textbf{68.2}&	OOM&	56.4&	55.4&	57.5&	86.2&	80.0&	83.5&	OOM&	--& --   \\
             & (1.3)&	(3.7)&	(2.2)&	--&	(8.3)&	(1.4)&	(0.6)&	(0.9)&	(1.4)&	(1.2)& -- & -- & --\\
             \normalcolor
             
            \textit{\textbf{AdaMeZO}} & \textbf{92.7}&	\textbf{63.0}&	67.8&	\textbf{70.6}&	\textbf{58.4}&	\textbf{55.8}&	\textbf{58.3}&	\textbf{87.0}&	80.1&	\textbf{83.7}&	\textbf{31.0} & \textbf{68.0} & \textcolor{black}{\textbf{70.4}} \\
             & (0.5)&	(6.1)&	(2.5)&	(3.6)&	(7.5)&	(0.6)&	(0.3)&	(1.1)&	(0.7)&	(1.3)&	(0.9) &-- & \textcolor{black}{--} \\
            \bottomrule
        \end{tabular}
    \end{center}
\end{table*}

\begin{figure}[htbp]
  \centering 
  
  \begin{subfigure}[b]{0.8\textwidth}
    \includegraphics[width=\linewidth]{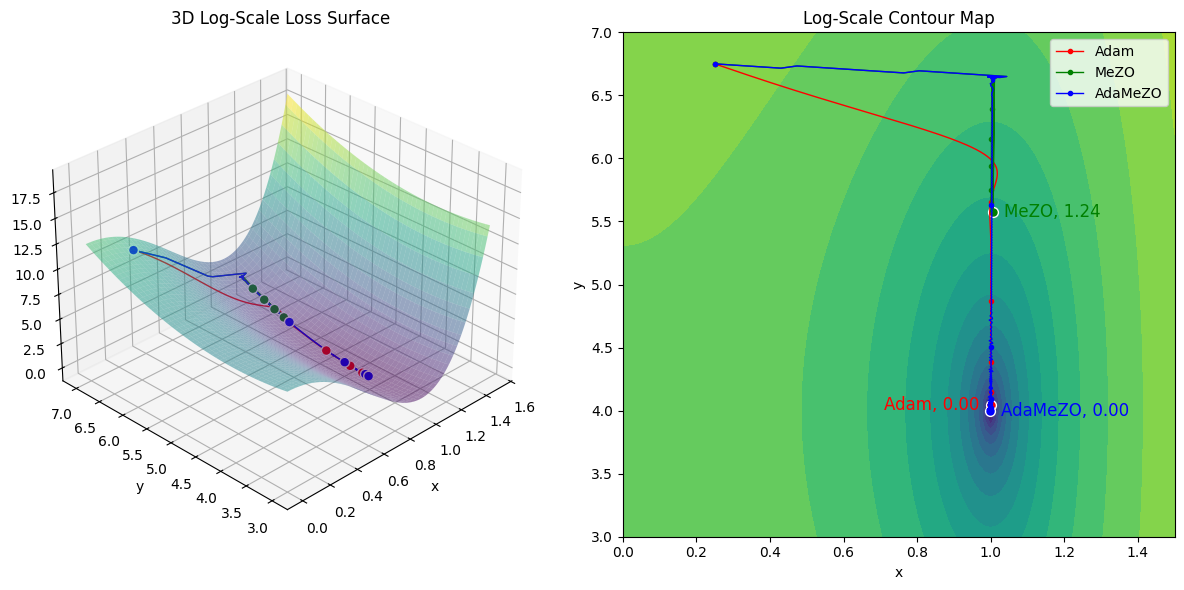} 
    \caption{Toy function $f_1(x, y) = 8 (x-1) ^ 2 (1.3 x ^2 + 2 x + 1) + 0.5 (y - 4) ^ 2$}
    \label{fig:sub1a}
  \end{subfigure}
  
  \begin{subfigure}[b]{0.8\textwidth}
    \includegraphics[width=\linewidth]{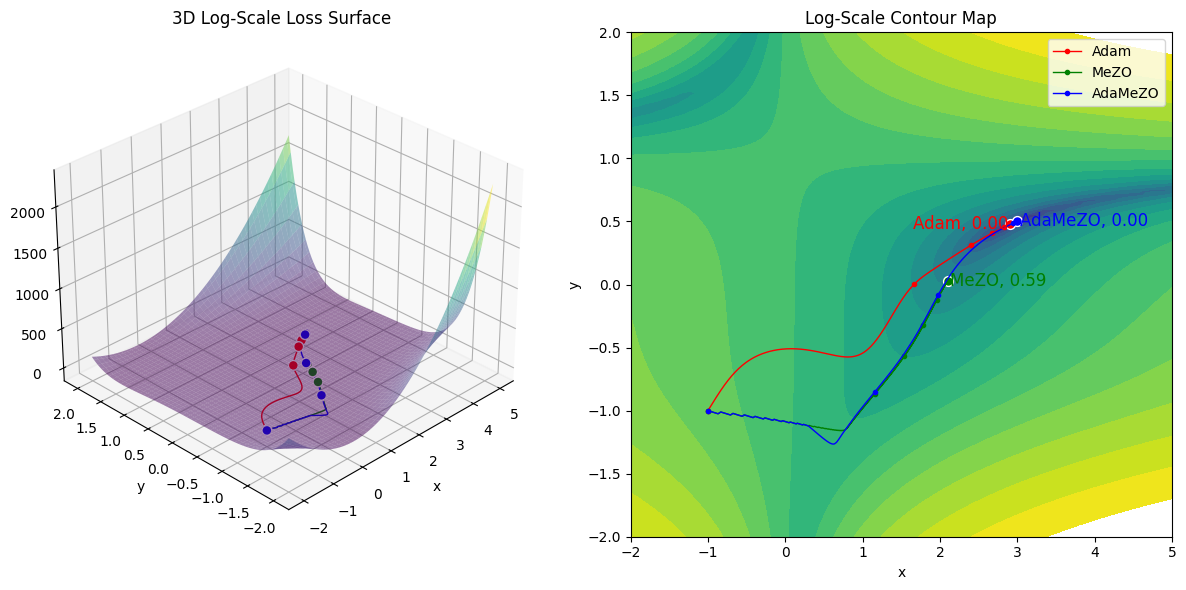}
    \caption{Toy function $f_2(x, y) = (1.5 - x + xy)^2 + (2.25 - x + xy^2)^2 + (2.625 - x + xy^3)^2$}
    \label{fig:sub2a}
  \end{subfigure}
  
  \begin{subfigure}[b]{0.8\textwidth}
    \includegraphics[width=\linewidth]{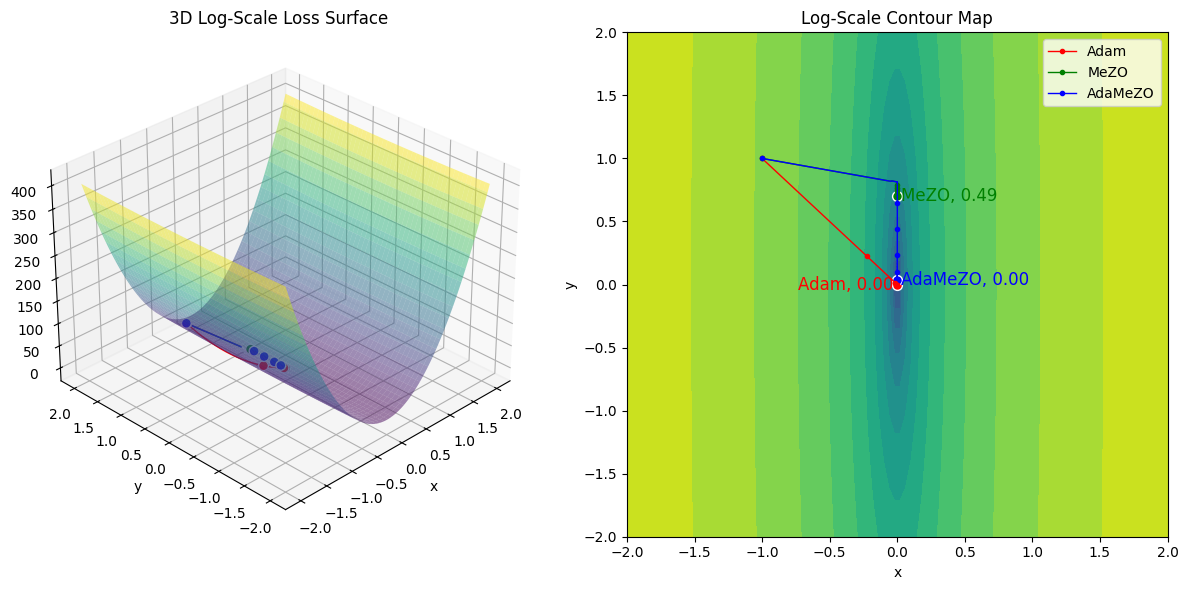}
    \caption{Toy function $f_3(x, y) = 100 x^2 + y ^ 2$}
    \label{fig:sub3a}
  \end{subfigure}

  \caption{Loss landscapes of the toy functions and optimization trajectories.} 
  \label{fig:bigtoy}
\end{figure}

\clearpage

\begin{figure}[h]
    \includegraphics[width=\linewidth]{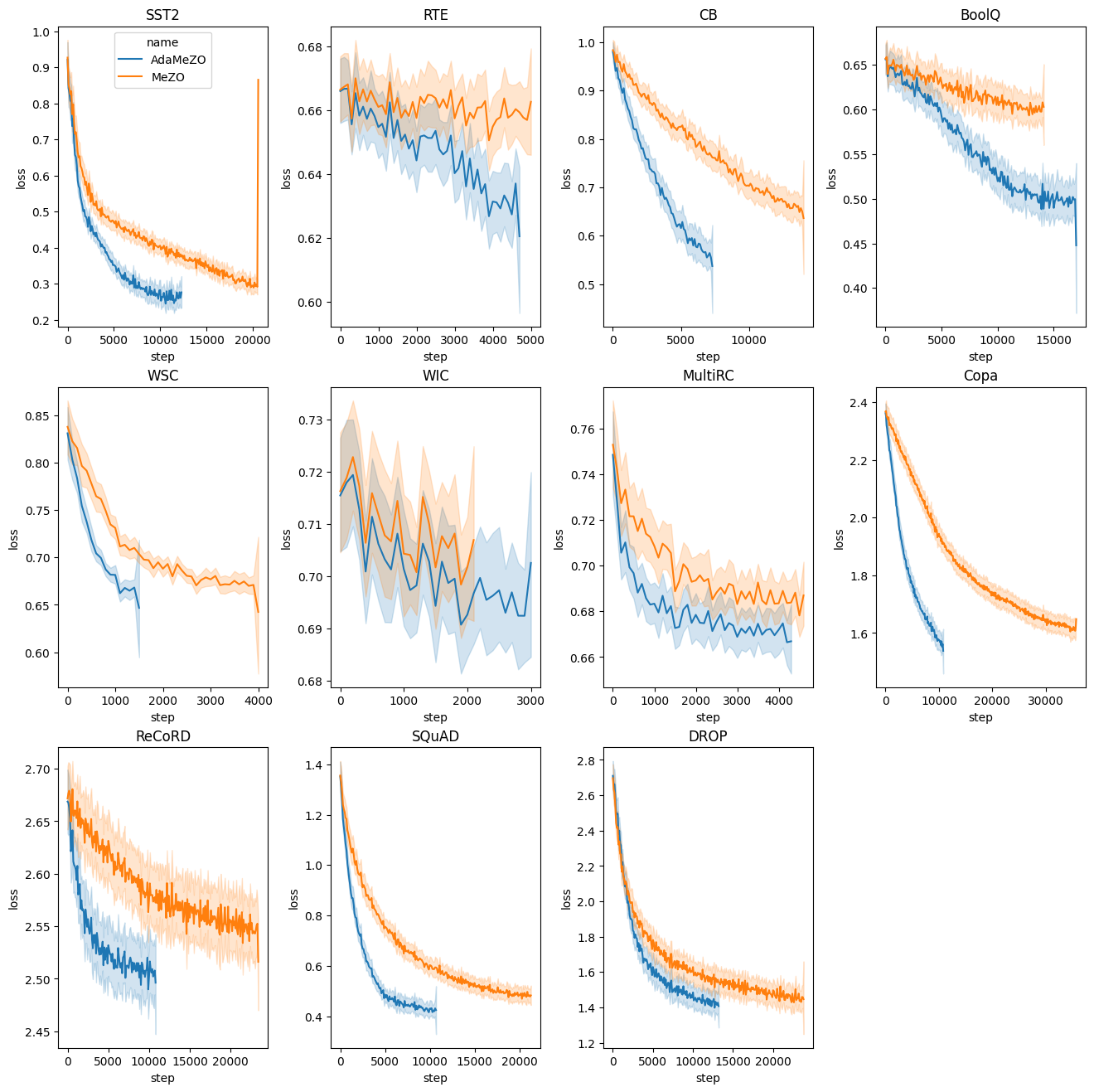} 
    \caption{Training loss curve of OPT-13B over language tasks.}
    \label{fig:trainloss13b}
\end{figure}

\begin{figure}[h]
    \includegraphics[width=\linewidth]{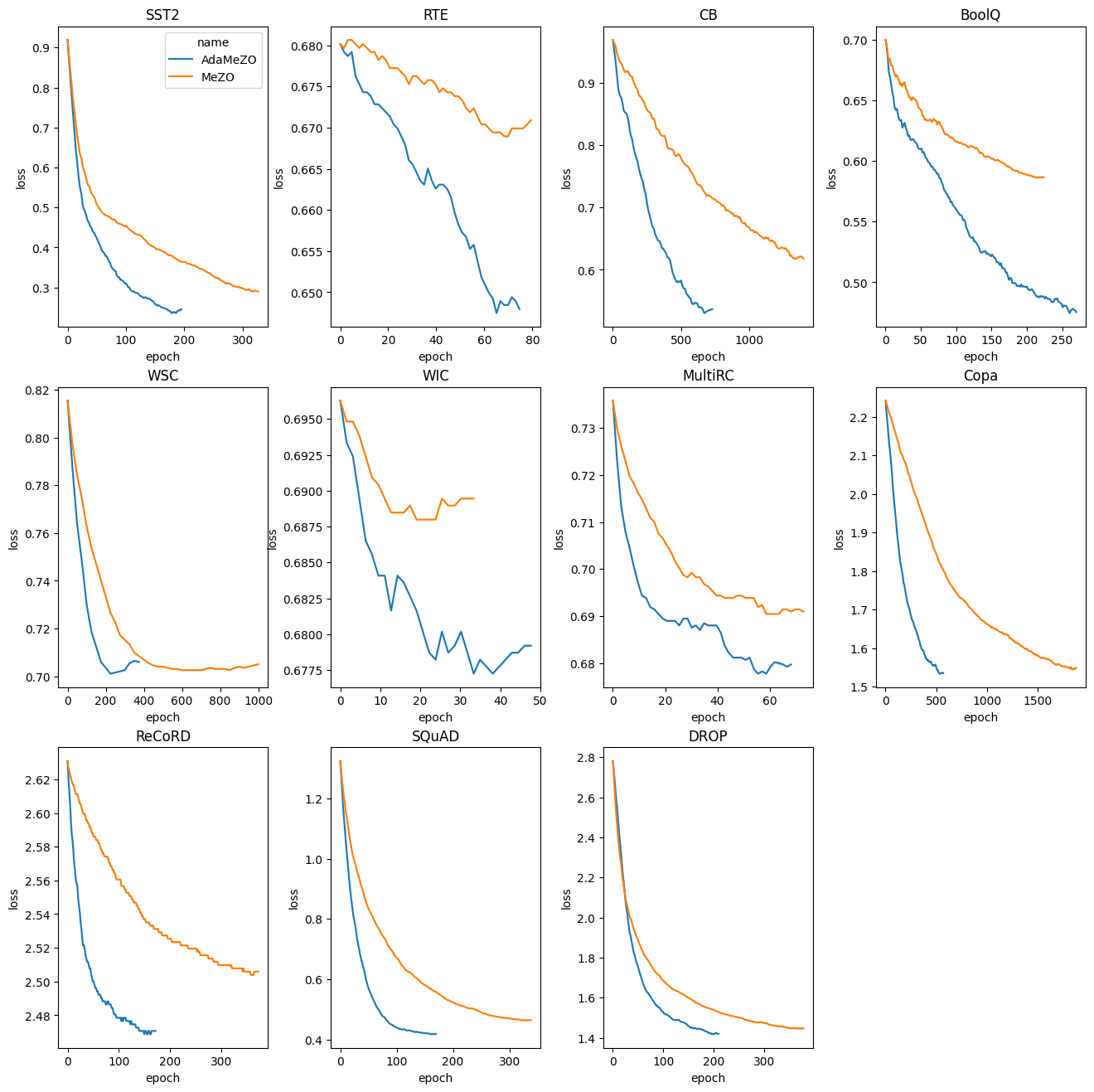} 
    \caption{Evaluation loss curve of OPT-13B over language tasks.}
    \label{fig:evalloss13b}
\end{figure}

\clearpage

\begin{table}[htbp]
    \centering
    \label{table:hyper_prefix}  
    \caption{\textcolor{black}{Hyperparameter settings for \Cref{table:nlp_opt_results_opt30b}.}}
    \color{black}
    \small
    \begin{tabular}{lcc}  
    \toprule  
    Experiment    & Hyperparameters & Values \\
    \midrule  
    \midrule
    \multirow{4}{*}{HiZOO (prefix)} & $B$      & 16 \\
                                   & $\eta$   & $\{5\mathrm{e}{-2}, 1\mathrm{e}{-2}, 5\mathrm{e}{-3}\}$ \\
                                   & $\mu$      & $1\mathrm{e}{-1}$ \\
                                   & \# prefix tokens& 5 \\
    \midrule
    \multirow{4}{*}{AdaMeZO (prefix)} & $B$      & 16 \\
                                   & $\eta$   & $\{7.5\mathrm{e}{-6}, 1\mathrm{e}{-5}, 2.5\mathrm{e}{-5}\}$ \\
                                   & $\mu$      & $1\mathrm{e}{-1}$ \\
                                   & \# prefix tokens& 5 \\
    \bottomrule  
    \end{tabular}
    \normalcolor
\end{table}

\begin{table}[htbp]
    \centering
    \label{table:hyper_prefix}  
    \caption{\textcolor{black}{Hyperparameter settings for HiZOO in \Cref{table:nlp_opt_results_opt_1_3b}, \Cref{table:nlp_opt_results_llama3b}, \Cref{table:nlp_opt_results_llama7b}, 
    and \Cref{table:nlp_opt_results_opt13b}.}}
    \color{black}
    \small
    \begin{tabular}{lcc}  
    \toprule  
    Experiment    & Hyperparameters & Values \\
    \midrule  
    \midrule
    \multirow{3}{*}{HiZOO} & $B$      & 16 \\
                                   & $\eta$   & $\{1\mathrm{e}{-6}, 5\mathrm{e}{-7}, 1\mathrm{e}{-7}\}$ \\
                                   & $\mu$      & $1\mathrm{e}{-3}$ \\
    \bottomrule  
    \end{tabular}
    \normalcolor
\end{table}

\clearpage

\begin{algorithm}[h]
\caption{AdaMeZO}\label{alg:adamezo}
\begin{algorithmic}
   \STATE {\bf Input:} Initialized model parameters $\boldsymbol{w}_0 \in \mathbb{R}^d$, loss function $\mathcal{L}: \mathbb{R}^d \to \mathbb{R}$, step budget $T$, perturbation scale $\mu$, learning rate $\eta$, horizon $h$, first EMA ratio $\beta_1$, second EMA ratio $\beta_2$, block strategy $B(\bs{w}) = \{\bs{w}^{(1)}, \dots, \bs{w}^{(b)}\}$, cancel factor $\beta_v$, warm-up steps $T_w$
   \STATE {\bf Output:} Trained model parameters $\boldsymbol{w}_T$

   \STATE {\tt seeds, projs} $\gets$ {\tt [], []}
   \FOR {$t = 1, \dots, T$}
   \STATE Sample batch $\mc{B}_t$ and random seed $s$
   \STATE Reset the PRNG with random seed $s$, spawn $\bs{z}_t \sim \mc{N}(\bs{0}, I_d)$
   \STATE Estimate $p_t$ using \Cref{eq:spsa} \COMMENT{in-place model perturbation}
   \STATE {\tt seeds.append(}$s${\tt)}, {\tt projs.append(}$p_t${\tt)}
   \STATE $\bs{w}_t \gets \bs{w}_t$
   \IF {$t > T_w$}
   \STATE {\tt states} $\gets$ {\tt [None]*} $(h, b)$
   \FOR {$\tau_b = 1, \dots, b$}
   \STATE $\bs{m}, \bs{v} \gets \bs{0}, \bs{0}$
   \FOR {$\tau_h = 1, \dots, h$}
   \STATE $p \gets {\tt projs}[-\tau_h]$
   \IF {{\tt states[}$\tau_h, \tau_b${\tt ] == None}}
   \STATE $s \gets {\tt seeds}[-\tau_h]$
   \STATE Reset the PRNG with random seed $s$, spawn $\bs{z} \sim \mc{N}(\bs{0}, I_{|\bs{w}^{(\tau_b)}|})$
   \ELSE
   \STATE Load {\tt states[}$\tau_h, \tau_b${\tt ]} to PRNG, spawn $\bs{z} \sim \mc{N}(\bs{0}, I_{|\bs{w}^{(\tau_b)}|})$
   \ENDIF
   \STATE Save PRNG state to {\tt states[}$\tau_h, \tau_b${\tt ]}
   \STATE $\bs{m} \gets \bs{m} + \beta_1^{\tau_h-1}p\bs{z}$
   \STATE $\bs{v} \gets \bs{v} + \beta_2^{\tau_h-1}p^2 (\bs{z} \odot \bs{z})$
   \ENDFOR
   \ENDFOR
   \STATE $\bs{w}_t^{(\tau_b)} \gets \bs{w}_{t}^{(\tau_b)} - \eta \beta_v \frac{\bs{m}}{\sqrt{\bs{v} + \epsilon}}$
   \ELSE
   \STATE Reset the PRNG with random seed $s$, spawn $\bs{z} \sim \mc{N}(\bs{0}, I_d)$
   \STATE $\bs{w}_t \gets \bs{w}_t - \eta p_t \boldsymbol{z}$
   \ENDIF
   \ENDFOR
\end{algorithmic}
\end{algorithm}

\clearpage
\section{Detailed Convergence Analysis}\label{sec:fullproof}

\begin{lemma}[Update expectations]\label{ge}
Given Assumption \ref{ass:bgv} to \ref{ass:fgd} and Assumption \ref{lm:magnus}, for warm-up steps, it holds that \begin{align}
    \bb{E}[\bs{u}_t] &= \nabla\mc{L}(\bs{w}_t) + \mc{O}(\mu), \label{eq:warm-up1}\\
    \bb{E}[\|\bs{u}_t\|_2^2] &\leq \frac{\eta^2 L \mc{O}(r)}{2} (\|\nabla \mc{L}(\bs{w}_t)\|_2^2 + \sigma^2) + \mc{O}(\mu^2). \label{eq:warm-up2}
\end{align}
After warm-up steps, it holds that \begin{align}
    \bb{E}[\bs{u}_t] &= \Sigma_t^{-1} (\nabla\mc{L}(\bs{w}_t) + \mc{O}(\bar{\beta}_1 L \eta)) + \mc{O}(\mu), \label{eq:pwarm-up1} \\
    \bb{E}[\|\bs{u}_t\|_2^2] &\leq (2{\rm tr}(\Sigma_t^{-1}) +4 s_l^{-1}) (\|\nabla\mc{L}(\boldsymbol{w}_t)\|_{\Sigma_t^{-1}}^2 + \sigma^2 + \mc{O}(\bar{\beta}_1^2 L^2 \eta^2)) + \mc{O}(\mu^2), \label{eq:pwarm-up2}
\end{align}
where $\sigma$ captures the batch stochasticity in first-order, $\bar{\beta}_1 = 1 - \beta_1$, and $\Sigma_t$ is the diagonal matrix with $\boldsymbol{v}_t$ being its diagonal.
\end{lemma}

\begin{proof}
The bounds for the warm-up phase follow Proof of Lemma 2 in \cite{malladi2023fine}. 

After the warm-up case, by the definition of $\bs{u}_t$, we have \begin{equation*}\begin{aligned}
    \bs{u}_t = \Sigma_t^{-1}\bs{m}_t,
\end{aligned}\end{equation*}
where \begin{align}
    \Sigma_t := \beta_v \sqrt{ \sum_{i=0}^{h-1} \beta_2^i\text{diag}\left(\frac{1}{n}\sum_{j=1}^n \bs{g}_{t-i,j} \odot \bs{g}_{t-i,j}\right)}, \quad \bs{g}_{t,j} := \bs{z}^\top_j\nabla\mc{L}(\bs{w}_t, \mc{B}_t)\bs{z}_j, \notag
\end{align}
with $\beta_v$ is a normalizing factor connected to $\beta_1$ and $\beta_2$ to cancel out all $\beta_1$ and $\beta_2$ related terms.

Moreover, \begin{align*}
    & \bb{E}\left[{\|\bs{u}_t\|_2^2}\right] \\
    =& \bb{E}_{\mc{B}_t, \bs{z}_j}\left[\|\frac{1}{n}\sum_{j=1}^n\Sigma_t^{-\frac12}\bs{z}_j\bs{z}_j^\top\Sigma^{-\frac12}(\nabla{\mc{L}(\bs{w}_t, \mc{B}_t)} + \mc{O}(\bar{\beta}_1 L \eta)) + \mc{O}(\mu)\|_2^2\right] \\
    \overset{(a)}{\leq}& 2 \bb{E}_{\mc{B}_t, \bs{z}}\left[\|\frac1n \sum_{j=1}^n \Sigma_t^{-\frac12} \bs{z}_j \bs{z}_j^\top \Sigma_{t}^{-\frac12}(\nabla{\mc{L}(\bs{w}_t, \mc{B}_t)} + \mc{O}(\bar{\beta}_1 L \eta))\|_2^2\right] + \mc{O}(\mu^2) \\
    \overset{(b)}{\leq}& \frac2n \sum_{j=1}^n\bb{E}_{\mc{B}_t, \bs{z}_j}\left[\|\Sigma_t^{-\frac12}\bs{z}_j\bs{z}_j^\top\Sigma^{-\frac12}(\nabla{\mc{L}(\bs{w}_t, \mc{B}_t)} + \mc{O}(\bar{\beta}_1 L \eta))\|_2^2\right] + \mc{O}(\mu^2) \\
    \overset{(c)}{=}& 2 {\rm tr}(\Sigma_t^{-1})\bb{E}_{\mc{B}_j}\left[(\nabla{\mc{L}(\bs{w}_t, \mc{B}_t)} + \mc{O}(\bar{\beta}_1 L \eta))^\top\Sigma_t^{-1}(\nabla{\mc{L}(\bs{w}_t, \mc{B}_t)} + \mc{O}(\bar{\beta}_1 L \eta))\right] \\
    &+ 4 \bb{E}_{\mc{B}_j}\left[(\nabla{\mc{L}(\bs{w}_t, \mc{B}_t)} + \mc{O}(\bar{\beta}_1 L \eta))^\top\Sigma_t^{-2}(\nabla{\mc{L}(\bs{w}_t, \mc{B}_t)} + \mc{O}(\bar{\beta}_1 L \eta))\right] + \mc{O}(\mu^2) \\
    \overset{(d)}{\leq}& (2{\rm tr}(\Sigma_t^{-1}) + 4s_l^{-1}) \bb{E}_{\mc{B}_j}\left[(\nabla{\mc{L}(\bs{w}_t, \mc{B}_t)} + \mc{O}(\bar{\beta}_1 L \eta))^\top\Sigma_t^{-1}(\nabla{\mc{L}(\bs{w}_t, \mc{B}_t)} + \mc{O}(\bar{\beta}_1 L \eta))\right] + \mc{O}(\mu^2) \\
    \overset{(e)}{\leq}& (2{\rm tr}(\Sigma_t^{-1}) + 4s_l^{-1}) (\|\nabla \mc{L}(\bs{w}_t)\|_{\Sigma^{-1}}^2 + \mc{O}(\bar{\beta}_1 L \eta) + s_l^{-1}\sigma_t^2) + \mc{O}(\bar{\beta}_1 L \eta).
    \end{align*}
where $(a)$ is by $\|a+b\|_2^2 \leq \|a\|_2^2 + \|b\|_2^2 + 2\|ab\|_2 \leq 2\|a\|_2^2 + 2 \|b\|_2^2$; $(b)$ is by the convexity of the function $\|\cdot\|^2$; $(c)$ is by setting $A = \bb{E}_{\mc{B}_j}\left[\Sigma_t^{-\frac12}(\nabla{\mc{L}(\bs{w}_t, \mc{B}_t)} + \mc{O}(\bar{\beta}_1 L \eta))^\top(\nabla{\mc{L}(\bs{w}_t, \mc{B}_t)} + \mc{O}(\bar{\beta}_1 L \eta))\Sigma_t^{-\frac12}\right]$ and $B = \Sigma_t^{-1}$, then apply Assumption \ref{lm:magnus}; $(d)$ is by Assumption \ref{ass:bsm}; finally $(e)$ by Assumption \ref{ass:bgv}.
\end{proof}

Finally, we establish \Cref{thm:meanqua}.

\begin{proof}
Split the full summation into the warm-up phase and the post-warm-up phase as follows.
\begin{align*}
    \frac{1}{T}\sum_{t=1}^T\|\nabla\mc{L}(\bs{w}_t)\|_2^2 = \frac{1}{T} \underbrace{\sum_{t=1}^{T_w} \|\nabla\mc{L}(\bs{w}_t)\|_2^2}_{{\rm warm-up}} + \frac{1}{T} \sum_{t=T_w+1}^T \|\nabla\mc{L}(\bs{w}_t)\|_2^2.
\end{align*}

Choose \begin{align*}
    \eta \leq \min \left\{\frac{1}{s ({\rm tr} \Sigma_{t}^{-1} + 2 s_l^{-1})\sqrt{T}}, \frac{1}{L \mc{O}(r) \sqrt{T}}, \frac{1}{s\bb{E}[\mc{L}(\bs{w}_1)] - \bb{E}[\mc{L}(\bs{w}_T)]\sqrt{T}}\right\},
\end{align*}
\Cref{eq:warm-up1} and \eqref{eq:warm-up2} with Assumption \ref{ass:lsmooth} yields \begin{align*}
    \bb{E}[\mc{L}(\bs{w}_{t+1})] &\leq \mc{L}(\bs{w}_t) - \eta \|\nabla \mc{L}(\bs{w}_t)\|_2^2 + \frac{\eta^2 L \mc{O}(r)}{2} (\|\nabla\mc{L}(\bs{w}_t)\|_2^2 + \sigma^2) + \mc{O}(\mu^2) \\
    &\leq \mc{L}(\bs{w}_t) - \frac{\eta}{2} \|\nabla \mc{L}(\bs{w}_t)\|_2^2 + \frac{\eta^2 L\sigma^2\mc{O}(r)}{2} + \mc{O}(\mu^2).
\end{align*}

\Cref{eq:pwarm-up1} and \eqref{eq:pwarm-up2} with Assumption \ref{ass:lsmooth} yields \begin{align*}
    \bb{E}[\mc{L}(\bs{w}_{t+1})] &\leq \mc{L}(\bs{w}_t) - \eta \|\nabla \mc{L}(\bs{w}_{t})\|_{\Sigma_t^{-1}}^2 + L \eta^2 ({\rm tr} \Sigma_t^{-1} + 2 s_l^{-1}) \left(\|\nabla \mc{L}(\bs{w}_t)\|_{\Sigma_t^{-1}}^2 + \mc{O}(\bar{\beta}_1^2 L^2 \eta^2) + \sigma^2\right) + \mc{O}(\mu^2) \\
    &\leq \mc{L}(\bs{w}_{t}) - \frac{\eta}{2} \|\nabla \mc{L}(\bs{w}_t)\|_{\Sigma_t^{-1}}^2 + L \eta^2 (\mc{O}(\bar{\beta}_1^2 L^2 \eta^2) + \sigma^2) ({\rm tr}\Sigma_t^{-1} + 2 s_l^{-1}) + \mc{O}(\mu^2).
\end{align*}

So, for the warm-up phase, \begin{align}
    \frac{1}{T} \sum_{t=1}^{T_w} \|\nabla \mc{L}(\bs{w}_t)\|_2^2 \leq \frac{2}{\eta T}(\mc{L}(\bs{w}_1) - \bb{E}[\mc{L}(\bs{w}_{T_w})]) + \frac{T_w L \eta \sigma^2 \mc{O}(r)}{T} + \mc{O}(\mu^2), \label{eq:warmup_e}
\end{align}

and for the post-warm-up phase, \Cref{eq:pwarm-up1} and \eqref{eq:pwarm-up2} with Assumption \ref{ass:lsmooth} yields \begin{align}
    &\frac{1}{T} \sum_{t=T_w+1}^{T} \|\nabla \mc{L}(\bs{w}_t)\|_2^2  \notag \\
    \leq& \frac{s_u}{T} \sum_{t=T_w+1}^T \|\nabla\mc{L}(\bs{w}_t)\|_{\Sigma_t^{-1}}^2  \notag \\
    \leq& \frac{2 s_u}{\eta T}(\bb{E}[\mc{L}(\bs{w}_{T_w+1})] - \bb{E}[\mc{L}(\bs{w}_T)]) + \frac{s_u (T-T_w)L \eta(\mc{O}(\bar{\beta}_1^2 L^2 \eta^2) + \sigma^2) ({\rm tr} \Sigma_{t}^{-1} + 2 s_l^{-1})}{T} + \mc{O}(\mu^2)\label{eq:pwarmup_e}.
\end{align}

Take $s = \max\{1, s_u\}$, combine \Cref{eq:warmup_e} and \eqref{eq:pwarmup_e}, \begin{align*}
    e \leq& \frac{2 s}{\eta T}(\bb{E}[\mc{L}(\bs{w}_{1})] - \bb{E}[\mc{L}(\bs{w}_{T})]) + \frac{T_w \eta L \sigma^2 \mc{O}(r)}{T} + sL \eta (\mc{O}(\bar{\beta}_1^2 L^2 \eta^2) + \sigma^2)({\rm tr}\Sigma_t^{-1} + 2 s_l^{-1}) + \mc{O}(\mu^2) \\
    \leq& \frac{L \sigma^2}{\sqrt{T}} + \frac{2}{T \sqrt{T}} + \frac{T_w (\sigma^2 + \mc{O}(\bar{\beta}_1^2 L^2 \eta^2))}{T\sqrt{T}} + \mc{O}(\mu^2).
\end{align*}
We omit the higher order terms to arrive at the target.
\end{proof}


\end{document}